\documentclass[journal]{IEEEtran}
\usepackage{cite}

\ifCLASSINFOpdf
  \usepackage[pdftex]{graphicx}
\else
\fi
\usepackage{amsmath}
\usepackage{amsthm}
\usepackage{amssymb}
\usepackage[]{algorithm2e}
\usepackage{subfig}
\usepackage{xcolor}
\usepackage{hyperref}

\begin{document}

\bstctlcite{BSTcontrol}
%
\title{Learning a Low-dimensional Representation of a Safe Region for Safe Reinforcement Learning on Dynamical Systems}
%
%
%

\author{Zhehua~Zhou, 
        Ozgur S.~Oguz, 
        Marion~Leibold, 
        and~Martin~Buss 
\thanks{This paper has been accepted by IEEE Transactions on Neural Networks and Learning Systems. DOI: 10.1109/TNNLS.2021.3106818}
\thanks{Z. Zhou, M. Leibold and M. Buss are with the Chair of Automatic Control Engineering, Technical
University of Munich, Munich 80290, Germany (e-mail: zhehua.zhou@tum.de; marion.leibold@tum.de; mb@tum.de).}
\thanks{O. Oguz is with the Max Planck Institute for Intelligent Systems and University of Stuttgart (e-mail: ozgur.oguz@ipvs.uni-stuttgart.de).}}

%
%

\markboth{IEEE Transactions on Neural Networks and Learning Systems}%
{}
%



\maketitle

\begin{abstract}
For the safe application of reinforcement learning algorithms to high-dimensional nonlinear dynamical systems, a simplified system model is used to formulate a safe reinforcement learning framework.
Based on the simplified system model, a low-dimensional representation of the safe region is identified and used to provide safety estimates for learning algorithms.
However, finding a satisfying simplified system model for complex dynamical systems usually requires a considerable amount of effort.
To overcome this limitation, we propose a general data-driven approach that is able to efficiently learn a low-dimensional representation of the safe region.
By employing an online adaptation method, the low-dimensional representation is updated using the feedback data to obtain more accurate safety estimates.
The performance of the proposed approach for identifying the low-dimensional representation of the safe region is illustrated using the example of a quadcopter.
The results demonstrate a more reliable and representative low-dimensional representation of the safe region compared to previous work, which extends the applicability of the safe reinforcement learning framework.
\end{abstract}


%
\IEEEpeerreviewmaketitle

\section{Introduction}
%
%
%
%
\IEEEPARstart{R}{ecent} studies of applying reinforcement learning or deep reinforcement learning algorithms to complex, i.e., highly nonlinear and high-dimensional, dynamical systems have demonstrated attractive achievements in various control tasks, e.g., humanoid control~\cite{peng2017deeploco} and robotic manipulator control~\cite{levine2016end}.
However, although the results display the potential of utilizing reinforcement learning algorithms as a substitute for traditional controller design techniques, most of them are still only presented in simulations~\cite{duan2016benchmarking}.
One major impediment against implementing reinforcement learning algorithms on real-world dynamical systems is that, due to the random exploration mechanism, the intermediate policy may lead to dangerous behaviors of the system.
As a result, both the system itself and the environment may be damaged during learning.
In order to apply state-of-the-art reinforcement learning algorithms to real-world control systems, one central problem to address is how to introduce a reliable safety guarantee into the learning process. 

\subsection{Related Work}
\label{sec.related_work}
Safe reinforcement learning (SRL) aims to find an optimal control policy by way of reinforcement learning while ensuring that certain safety conditions are not violated during the learning process.
Although the exact definition of safety in SRL varies in different learning tasks, for instance collision avoidance in autonomous vehicles or crash prevention when controlling a quadcopter, we generally consider the safety condition as neither the system itself nor the environment will be damaged.

SRL in dynamical systems with continuous action space has been a topic of research for over a decade~\cite{garcia2015comprehensive}.
Most previous studies employed a manual control mechanism to ensure the safety of the controlled system.
For instance, in~\cite{abbeel2007application}, an experienced human pilot takes over the control of the helicopter if the learning algorithm places the system in a risky state.
However, such an approach requires a considerable amount of resource to monitor the entire learning process. 
Hence, in most cases, it is not applicable to complex learning tasks.
Another possibility of safely implementing reinforcement learning algorithms on real-world dynamical systems is by transfer learning~\cite{pan2010survey}.
First, a satisfying initial policy is trained in simulation and then transferred to the real-world dynamical system.
In essence, this minimizes required number of learning iterations for obtaining the final policy and thus reduces the risk of encountering dangerous intermediate policy~\cite{christiano2016transfer}. 
However, since the mismatch between simulation and reality is not considered in transfer learning, no reliable safety guarantee is obtained~\cite{huang2017adversarial}. 

In recent studies, SRL in model-free scenarios is usually achieved by solving a constraint satisfaction problem.
For example, constrained policy optimization~\cite{achiam2017constrained} introduces a constraint to the learning process to the effect that the expected return of cost functions should not exceed certain predefined limits.
Alternatively, including an additional risk term in the reward function, such as risk-sensitive reinforcement learning~\cite{shen2014risk}, can also increase the safety of reinforcement learning algorithms.
However, as no system model is directly considered in these approaches, there is still a high possibility that safety conditions are violated, especially in the early learning phase.

When at least an approximated system model is available, a more promising SRL can be realized by combining control-theoretic concepts with reinforcement learning approaches.
For example in~\cite{perkins2002lyapunov,chow2018lyapunov}, Lyapunov functions are employed to compute a sub-region of the state space where safety conditions will never be violated.
The system is then limited to this sub-region during the learning process.
However, finding suitable candidates for Lyapunov functions is challenging if the system dynamics contains uncertainties or is highly nonlinear.

For uncertain dynamical systems, methods based on learning a model of unknown system dynamics~\cite{ostafew2016robust} or of environmental constraints~\cite{sadigh2016safe} are proposed to ensure safety during learning.
For instance, by predicting the system behavior in the worst case, robust model predictive control~\cite{zanon2020safe} is able to provide safety and stability guarantees to reinforcement learning algorithms if the error in the learned model is bounded. 
Besides, \cite{li2021safe} introduces an action governor to correct the applied action when the system is predicted to be unsafe.
However, limited by computational efficiency, these approaches with deterministic safety estimates, i.e., the prediction about the safety of a system state is either safe or unsafe, are usually only applicable to linear systems.
Moreover, the accuracy of the learned model also strongly affects the performance of these approaches.

To relax the demands placed on the system model and extend the SRL to nonlinear systems, instead of deterministic safety estimates, recent studies employ probabilistic safety estimates, in which safety predictions are represented as probabilities~\cite{moldovan2012safe}.
In~\cite{fisac2018general}, for example, modelling uncertainties are approximated by Gaussian process models~\cite{ki2006gaussian}, and a probabilistic safe region is computed by reachability analysis~\cite{bansal2017hamilton}.
Similarly, Gaussian process models are used in~\cite{berkenkamp2016safe,berkenkamp2017safe} to model unknown system dynamics. 
A safe region is then obtained from the probabilistic estimate of the region of attraction (ROA) of a safe equilibrium state. 
The key component of these studies is a forward invariant safe region, such that the learning algorithm has the flexibility to execute desired actions within the safe region.
Safety is ensured by switching to a safety controller whenever the system approaches the boundary of the safe region.
However, the safe region is computed either by solving a partial differential equation in~\cite{fisac2018general} or sampling in~\cite{berkenkamp2017safe}, both of which suffer from the curse of dimensionality.
Moreover, modeling an unknown dynamics or disturbance with Gaussian process models also poses challenges when the system is highly nonlinear and high-dimensional, since both making adequate assumptions about the distribution of dynamics and acquiring a sufficient amount of data are difficult.
Therefore, although approaches like~\cite{fisac2018general,berkenkamp2017safe} enable promising results with low-dimensional dynamical systems\footnote{In this paper we consider dynamical systems with dimensions higher than six as high-dimensional, as in such cases it is computationally difficult to implement traditional methods, such as reachability analysis or sum-of-squares programming, in identifying the safe region.}, they are not directly applicable to complex dynamical systems~\cite{fisac2019bridging}.

Often the motivation for using reinforcement learning algorithms for controller design is to overcome the difficulty of applying model-based controller design approaches to highly nonlinear, high-dimensional and uncertain dynamic system models~\cite{mahmood2018benchmarking,james2020rlbench}. 
In particular, it is challenging to compute a safe region for a complex dynamical system.
For this reason, \cite{zhou2020general} introduces an SRL framework that utilizes a supervisory control strategy based on finding a simplified system by means of physically inspired model order reduction~\cite{schilders2008model}.
A simplified safe region is constructed from the simplified system, which functions as an approximation for the safe region of the full dynamics.
Such a low-dimensional representation of the safe region, which is usually two- or three-dimensional, at least provides safety estimates for the original system states, and it can be updated online during the learning process.
To account for the uncertainty in making safety decisions for the complex dynamics based on a rough low-dimensional reduction, the safety estimate is represented in a probabilistic form.
Then, in accordance with the derived safety estimate, a supervisor is employed to switch the actual applied control action between the learning algorithm and a corrective controller to keep the system safe. 
However, implementing physically inspired model order reduction usually requires a thorough understanding of the system dynamics.
Moreover, multiple performance tests are required before a satisfying simplified system can be found.

\subsection{Contribution}
In this paper, we consider the same supervisory control strategy as used in~\cite{zhou2020general} to construct a general SRL framework that is applicable to complex dynamical systems.
However, to overcome the limitations of physically inspired model order reduction, we propose a novel data-driven approach to identify the supervisor, i.e., the low-dimensional representation of the safe region.
Inspired by transfer learning~\cite{marco2017virtual}, we assume that an approximated system model of the complex dynamical system is available. 
Even though, inevitably, the approximated model displays discrepancies compared with the real system behavior, an initial estimate of safety can usually be obtained by simulating the approximated model.  
For example, while the dynamics of a real-world humanoid cannot be known perfectly, an approximated humanoid model can be constructed in simulation for making predictions.
Hence, by simulating the system, we obtain training data that represents the safety of various original system states.
However, as the state space is high-dimensional, it is infeasible to acquire a sufficient amount of training data to directly learn the safe region of the original system.
To solve this problem, a data-driven approach that computes probabilistic similarities between each training data is proposed to first learn a low-dimensional representative safety feature of the complex dynamical system.
Then, based on the learned feature, a low-dimensional representation of the safe region is identified, which is used as the starting point to SRL in the real system.

Due to the inevitable simulation-to-reality gap, the initial low-dimensional representation of the safe region learned from training data displays discrepancies compared to the real system behavior.
To compensate for this mismatch, we also propose an efficient online adaptation method to update the low-dimensional representation of the safe region.
During the learning process, we receive feedback data about the actual safe region of the real system.
These feedback data are not only used to generate new safety estimates, but they also allow us to adjust our confidence in the reliability of the safety estimates obtained from training data.  
The proposed online adaptation method then updates the low-dimensional representation of the safe region by simultaneously considering the safety estimates derived from training and feedback data.

The contributions of this study are summarized as follows:
\begin{enumerate}
\item We propose a novel data-driven approach that is capable of systematically identifying a low-dimensional representation of the safe region.
In contrast to physically inspired model order reduction, the proposed approach does not require a thorough understanding of system dynamics. 
Moreover, it is applicable to a wide range of dynamical systems, as long as an approximated system model is available.
    
\item We introduce a new online adaptation method for updating the low-dimensional representation of the safe region according to the observed real system behavior.
By fully utilizing the information contained in the feedback data, the update is performed efficiently, while a reasonable amount of feedback data enables an accurate low-dimensional representation of the safe region to be acquired. 

\item Since the proposed approach results in a reliable and representative low-dimensional representation of the safe region, the applicability of the SRL framework is increased. 

\end{enumerate}

The remainder of this paper is organized as follows: a brief introduction to the SRL framework is given in Section~\ref{sec.preliminary}.
Thereafter, we present an overview of our approach in Section~\ref{sec.problem_formulation}. 
In Section~\ref{sec.learning_representation}, we propose a data-driven method to derive a low-dimensional representation of the safe region. 
This is followed by the online adaptation method in Section~\ref{sec.online_adaptation}, which is used to update the low-dimensional representation.
An example is presented in Section~\ref{sec.results} to demonstrate the performance of the proposed approach.
In Section~\ref{sec.discussion}, we discuss several properties of the approach, and Section~\ref{sec.conclusion} concludes the paper.
A table of nomenclatures is included in the supplementary material.

\section{Safe Reinforcement Learning Framework}
\label{sec.preliminary}
In this paper, we consider SRL as to optimize a learning-based policy with respect to a predefined reward function, while ensuring that the system state remains in a safe region of the state space.
In this section, we outline a general SRL framework for dynamical systems, see also~\cite{zhou2020general}.
The SRL framework first identifies a safe state-space region as the safe region.
Then, the learning-based policy has the flexibility to execute desired actions within the safe region.
Once the system state is about to leave the safe region, a corrective controller is applied to drive the system back to a safe state.

\subsection{System Model and Safe Region}
A nonlinear control-affine dynamical system is given by
\begin{equation}
    \dot{x} = f(x) + g(x)u
\label{eq.system_nomial}
\end{equation}
where $x \in \mathcal{X} \subseteq \mathbb{R}^n$ is the $n$-dimensional system state within a connected set $\mathcal{X}$, $u \in \mathcal{U} \subseteq \mathbb{R}^m$ is the $m$-dimensional control input to the system. 
With a given control policy $u = K(x)$, the closed-loop system dynamics is denoted as 
\begin{equation}
    \dot{x} = f_K(x) = f(x) + g(x)K(x).
\label{eq.system_nomial_closed}
\end{equation}
If a system state $x$ satisfies $f_K(x) = 0$, then it is an equilibrium point. 
Any equilibrium point can be shifted to the origin by a state transform.
Therefore, this paper only uses the origin to formulate the safe region.

\newtheorem{assumption}{Assumption}
\begin{assumption}
\label{ass.system_property}
The origin is a safe state and a locally asymptotically stable equilibrium point under the control policy $K(x)$.
\end{assumption}

Based on Assumption~\ref{ass.system_property}, the ROA of the origin is defined as
\begin{equation}
    \mathcal{R} = \{ x_0 \in \mathcal{X} \hspace{1mm} | \hspace{1mm} \lim_{t \rightarrow \infty} \Phi(t;x_0) = 0 \}
\label{eq.ROA}
\end{equation}
where $\Phi(t;x_0)$ is the system trajectory of~(\ref{eq.system_nomial_closed}) that starts at the initial state $x_0$ when time $t = 0$.
The ROA $\mathcal{R}$ is the set of initial states that can be driven back to a safe state, i.e., the origin, under the control policy $K(x)$.
Therefore in this paper, we define the safe region of the SRL framework as follows. 
\newtheorem{definition}{Definition}
\begin{definition}
\label{def.safe_region}
A safe region $\mathcal{S}$ is a closed positive invariant subset of the ROA $\mathcal{R}$ containing the origin.
We consider the system state $x$ as safe if it is in the safe region $\mathcal{S}$.
\end{definition}

\subsection{SRL Framework}

\begin{figure}[!t]
\centering
\includegraphics[width=0.65\linewidth]{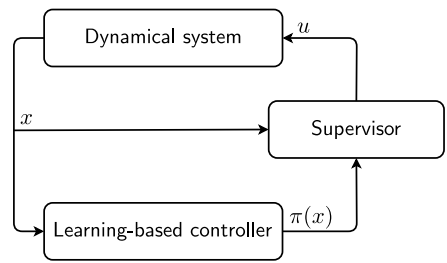}
\caption{SRL framework with a supervisor which decides on the actual applied actions.}
\label{fig.SRL_framework}
\end{figure}

To realize SRL, we keep the system state within the safe region during the learning process.
This is achieved by an SRL framework that adapts a switching supervisory control strategy where the given controller $K(x)$ acts as corrective control and $\pi(x)$ is the learning-based policy that is used while the system state is in the safe region (see Fig.~\ref{fig.SRL_framework}).
A supervisor determines the actual applied actions as
\begin{equation}
    u = \begin{cases} 
    \pi(x), & \text{if } t < t^{\mathrm{safe}} \\
    K(x), & \text{else} 
\end{cases} 
\label{eq.original_supervisor}
\end{equation}
where $t^{\mathrm{safe}}$ is the first time point at which the system state $x$ is on the boundary of the safe region $\mathcal{S}$.

For each learning iteration, the system starts inside the safe region $\mathcal{S}$ for time $t = 0$.
The learning algorithm then updates and executes the learning-based policy $\pi(x)$. 
Since the safe region $\mathcal{S}$ is a closed set and the trajectory is continuous, the system state can only leave the safe region $\mathcal{S}$ by crossing the boundary.
Hence, once the system state $x$ is on the boundary of the safe region $\mathcal{S}$, this learning iteration is terminated at time $t = t^{\mathrm{safe}}$ and the corrective controller $K(x)$ is activated.
For the remaining time of this learning iteration, the corrective controller $K(x)$ attempts to bring the system back to the origin to maintain safety.
After this safety recovery, the learning environment is reset and the next learning iteration starts at time $t = 0$.

\newtheorem{remark}{Remark}
\begin{remark}
\label{remark.safe_region}
In this paper, we only consider the safe region obtained from the ROA $\mathcal{R}$, where stability is used as the safety criterion.
If more safety criteria should be taken into consideration, such as collision avoidance represented as state constraints, the safe region can be constructed using other control-theoretical concepts, e.g., control barrier functions~\cite{romdlony2016stabilization} or invariance functions~\cite{sobotka2007invariance}.
The definition of the safe region does not affect the use of the SRL framework and the proposed approach, as long as the safe region is a closed and control invariant set under a given corrective controller. 
\end{remark}

\subsection{SRL Framework for Complex Dynamical Systems}

The aforementioned SRL framework is not directly applicable to complex dynamical systems, as in such cases, calculating the safe region $\mathcal{S}$ is computationally infeasible~\cite{ahmadi2019dsos}.
An SRL framework based on estimating safety with a low-dimensional representation of the safe region is introduced to overcome this problem~\cite{zhou2020general}. 

Each original system state $x$ is mapped to a low-dimensional safety feature, represented as a simplified state $y \in \mathcal{Y} \subseteq \mathbb{R}^{n_y}$, $n_y \ll n$, through a state mapping $y = \Psi(x)$.
The state mapping is chosen such that safe and unsafe states are separated in the simplified state space $\mathcal{Y}$.
Nevertheless, due to the order reduction, multiple original system states that have different safety properties can map to the same simplified state.
Hence, the safety of the original system state $x$ is estimated by the safety of its corresponding simplified state $y$ in a probabilistic form as
\begin{equation}
    p(x \in \mathcal{S}) = \Gamma(y)|_{y = \Psi(x)} \sim [0,1]
\label{eq.safety_gamma}
\end{equation}
where $\Gamma(y)$ is a function defined over the simplified state space $\mathcal{Y}$ and is referred to as the \textit{safety assessment function (SAF)} in this paper.
Not only does the SAF $\Gamma(y)$ encode information relating to the safety of the simplified state $y$, it also includes the uncertainty involved in making predictions for a high-dimensional state by using a low-dimensional reduction.
In Section~\ref{sec.learning_representation}, we demonstrate how to efficiently identify the state mapping $y = \Psi(x)$ as well as the SAF $\Gamma(y)$ using a data-driven method.

For a given SAF $\Gamma(y)$, the probability $p(x \in \mathcal{S})$ depends only on the simplified state $y$.
Therefore, by introducing a predefined probability threshold $p_t$, we obtain a low-dimensional representation of the safe region, denoted as $\mathcal{S}_y$, in the simplified state space $\mathcal{Y}$
\begin{equation}
\mathcal{S}_y = \{ y \in \mathcal{Y} \hspace{1mm} | \hspace{1mm} \Gamma(y) > p_t \}
\label{eq.safe_region_low}
\end{equation}
which works as an approximation of the high-dimensional safe region $\mathcal{S}$.
The supervisor~(\ref{eq.original_supervisor}) is thus modified to
\begin{equation}
     u = \begin{cases} 
    \pi(x), & \text{if } t < t^{\mathrm{safe'}} \\
    K(x), & \text{else }
\end{cases} 
\label{eq.prob_supervisor}
\end{equation}
where $t^{\mathrm{safe'}}$ denotes the first time point at which the probability $p(x \in \mathcal{S})$ is not larger than the threshold $p_t$, i.e., $p(x \in \mathcal{S}) = \Gamma (y) \leq p_t$.
More details of this SRL framework are given in~\cite{zhou2020general}.

\section{Overview of the Approach}
\label{sec.problem_formulation}
The essential factor when applying the SRL framework to complex dynamical systems is finding a reliable low-dimensional representation of the safe region $\mathcal{S}_y$.
In order to overcome the limitations of physically inspired model order reduction, we propose a novel data-driven approach to identify the low-dimensional representation of the safe region $\mathcal{S}_y$, together with a new online adaptation method to efficiently update the learned low-dimensional representation.

We consider a scenario in which the complex dynamical system, referred to as the real system, has partially unknown dynamics.
However, we assume that a nominal approximated system model is available and can be used to roughly predict the real system behavior.
The nominal system model is assumed to be represented by (\ref{eq.system_nomial}). 
The real system model is then given as
\begin{equation}
    \dot{x} = f(x) + g(x)u + d(x)
\label{eq.system_real}
\end{equation}
where $d(x)$ is the unknown, unmodelled part of the system dynamics.
For brevity, we refer to the nominal and the real systems as \textit{simulation} and \textit{reality}, respectively.

Due to the highly nonlinear and high-dimensional dynamics, the direct calculation of the safe region is computationally infeasible for both the nominal and the real systems. 
Besides, although the real system provides exact safety information, in general it is expensive to collect data directly on the real system.
In contrast, simulating the nominal system is usually efficient and allows a sufficient amount of data to be obtained for finding a low-dimensional safety representation.
However, due to the unknown term $d(x)$, such data is inaccurate and has to be modified to account for the real system behavior.

Based on these facts, to construct a reliable low-dimensional representation of the safe region $\mathcal{S}_y$ for the real system, we propose the approach outlined in Fig.~\ref{fig.overview} (a complete work-flow is given in the supplementary material).
It consists of two parts that solve the following two problems, respectively:
\begin{enumerate}
    \item How to derive and initialize the low-dimensional representation of the safe region $\mathcal{S}_y$ by using the nominal system model.
    \item How to update the low-dimensional representation of the safe region $\mathcal{S}_y$ online with the observed real system behavior.
\end{enumerate}

\begin{figure}[!t]
\centering
\includegraphics[width=.9\linewidth]{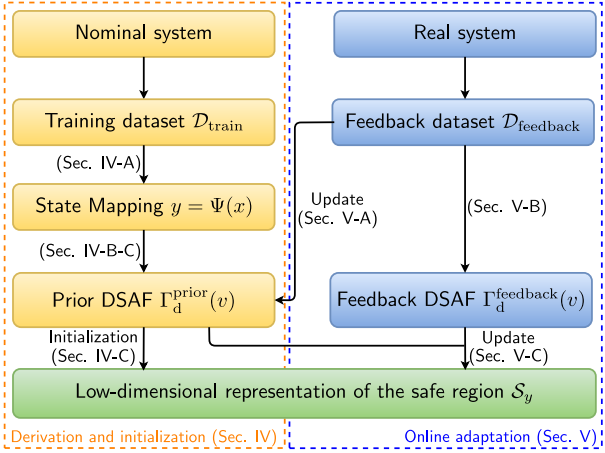}
\caption{Overview of the proposed approach. The low-dimensional representation $\mathcal{S}_y$ is initialized using the training dataset $\mathcal{D}_{\mathrm{train}}$ obtained from the nominal system. 
Once we collect the feedback dataset $\mathcal{D}_{\mathrm{feedback}}$ on the real system, the low-dimensional representation $\mathcal{S}_y$ is updated using the proposed online adaptation method.}
\label{fig.overview}
\end{figure}

\subsection*{Part 1) Derivation and Initialization}
Since no information about uncertainty $d(x)$ is available prior to the learning process, the corrective controller $K(x)$ is designed for the nominal system model (\ref{eq.system_nomial}).
Although the safe region of the nominal system is unknown, 
its simulation is possible and delivers a dataset as follows.

\begin{definition}
\label{def.training_data}
The training dataset of $k_t$ training data is given as
\begin{equation}
\mathcal{D}_{\mathrm{train}} = \{D_{\mathrm{train}}^1, D_{\mathrm{train}}^2, \ldots,D_{\mathrm{train}}^{k_t} \}.
\label{eq.training_dataset}
\end{equation}
It contains the simulation results that state whether the safety recovery is successful or not for different system states $x$ under the corrective controller $K(x)$.  
The $i$-th training data consists of three elements 
\begin{equation}
    D_{\mathrm{train}}^i = \{ x_{\mathrm{sim}}^i, s_{\mathrm{sim}}(x_{\mathrm{sim}}^i), \Phi_{\mathrm{sim}} (t;x_{\mathrm{sim}}^i) \}.
\label{eq.training_data}
\end{equation}
$x_{\mathrm{sim}}^i$ is the initial system state in which the corrective controller $K(x)$ is activated. $s_{\mathrm{sim}}(x_{\mathrm{sim}}^i)$ is the safety label that represents the result of safety recovery for the state $x_{\mathrm{sim}}^i$.  
We denote $s_{\mathrm{sim}}(x_{\mathrm{sim}}^i) = 1$ if the system state $x_{\mathrm{sim}}^i$ is safe under the corrective controller $K(x)$, and $s_{\mathrm{sim}}(x_{\mathrm{sim}}^i) = 0$ if it is not. 
$\Phi_{\mathrm{sim}} (t;x_{\mathrm{sim}}^i)$ is the corresponding system trajectory of the safety recovery that starts at $x_{\mathrm{sim}}^i$ when time $t=0$.
The subscript $\mathrm{sim}$ indicates that the data is collected by using the nominal system model.
\end{definition}

The low-dimensional representation of the safe region $\mathcal{S}_y$ is thus derived and initialized by using the training dataset $\mathcal{D}_{\mathrm{train}}$.
To do this, we first identify the state mapping $y = \Psi(x)$ using a data-driven method that computes the probabilistic similarity between each training data (Section~\ref{sec.tsne}).
Then to facilitate an efficient computation, we discretize the simplified state space $\mathcal{Y}$ into grid cells and assign an index vector $v \in \mathbb{Z}_{+}^{n_y}$ to each grid cell.
By assuming that the SAF $\Gamma(y)$ is constant in each grid cell, we thus obtain a \textit{discretized safety assessment function (DSAF)} $\Gamma_{\mathrm{d}}(v)$.  
A discretized low-dimensional representation of the safe region $\mathcal{S}_y$ is then given by applying the probability threshold $p_t$ on the DSAF $\Gamma_{\mathrm{d}}(v)$ (Section~\ref{sec.formulate_mapping}).    
To enable the SRL framework on the real system, we also calculate an initial estimate of the DSAF $\Gamma_{\mathrm{d}}(v)$, denoted as the prior DSAF $\Gamma_{\mathrm{d}}^{\mathrm{prior}}(v)$, from the training dataset $\mathcal{D}_{\mathrm{train}}$. 
It is then used to initialize the low-dimensional representation of the safe region $\mathcal{S}_y$ (Section~\ref{sec.initial_mapping}).
Further details of Part 1) are given in Section~\ref{sec.learning_representation}.

\subsection*{Part 2) Online Adaptation}
Due to the unknown part of the system dynamics $d(x)$, there is inevitably a mismatch between simulation and reality.
In order to compensate for this mismatch, we update the low-dimensional representation $\mathcal{S}_y$ by accounting for the real system behavior.

Each time the corrective controller $K(x)$ is activated during learning, we observe feedback data about the real safe region.
The set of feedback data is defined as follows.

\begin{definition}
\label{def.feedback_data}
The feedback dataset of $k_f$ feedback data is given as
\begin{equation}
    \mathcal{D}_{\mathrm{feedback}} = \{D_{\mathrm{feedback}}^1, D_{\mathrm{feedback}}^2,\ldots, D_{\mathrm{feedback}}^{k_f} \}.
\label{eq.feedback_dataset}
\end{equation} 
It contains the results of safety recovery from implementing the corrective controller $K(x)$ on the real system.
The $i$-th feedback data is 
\begin{equation}
    D_{\mathrm{feedback}}^i = \{ x_{\mathrm{real}}^i, s_{\mathrm{real}}(x_{\mathrm{real}}^i), \Phi_{\mathrm{real}} (t;x_{\mathrm{real}}^i) \}.
\label{eq.feedback_data}
\end{equation}
While $x_{\mathrm{real}}^i$, $s_{\mathrm{real}}(x_{\mathrm{real}}^i)$ and $\Phi_{\mathrm{real}} (t;x_{\mathrm{real}}^i)$ have the same meaning as in Definition~\ref{def.training_data}, the subscript $\mathrm{real}$ indicates here that the data is collected on the real system.
\end{definition}

Since collecting data on the real system, e.g., real-world robots, is usually expensive and time-consuming, in most cases the feedback dataset $\mathcal{D}_{\mathrm{feedback}}$ has a limited size.
Therefore, the low-dimensional representation of the safe region $\mathcal{S}_y$ needs to be updated in a data-efficient manner.
To achieve this, we propose an online adaptation method, as given in Section~\ref{sec.online_adaptation}.
It comprises three steps:
First, we modify the prior DSAF $\Gamma_{\mathrm{d}}^{\mathrm{prior}}(v)$ by changing our confidence in its reliability using the feedback dataset $\mathcal{D}_{\mathrm{feedback}}$ (Section~\ref{sec.update_prior}).
Second, to fully utilize the valuable information contained in the feedback dataset $\mathcal{D}_{\mathrm{feedback}}$, we generate another feedback DSAF $\Gamma_{\mathrm{d}}^{\mathrm{feedback}}(v)$ (Section~\ref{sec.update_feedback}).
Third, the two DSAFs are fused to obtain a more accurate DSAF $\Gamma_{\mathrm{d}}(v)$, which is then used to update the low-dimensional representation $\mathcal{S}_y$ (Section~\ref{sec.update_safety}).

\section{Learning a Low-dimensional Representation of the Safe Region}
\label{sec.learning_representation}
To derive the low-dimensional representation of the safe region $\mathcal{S}_y$, two components have to be determined: the state mapping $y = \Psi(x)$, which gives the low-dimensional safety feature, and the SAF $\Gamma(y)$, which predicts the safety of original system states.
In this section, we present a data-driven method for identifying the low-dimensional representation of the safe region $\mathcal{S}_y$.
It utilizes a technique called t-distributed stochastic neighbor embedding (t-SNE)~\cite{maaten2008visualizing}, which was originally proposed for visualizing high-dimensional data.

\subsection{Identifying the State Mapping with t-SNE}
\label{sec.tsne}
To identify the state mapping $y = \Psi(x)$, we first find the realization of the low-dimensional safety feature, i.e., the values of simplified states $y^1,\ldots, y^{k_t}$, that best corresponds with the training dataset $\mathcal{D}_{\mathrm{train}}$ by revising t-SNE. 
Through measuring the similarity between each high-dimensional data point, t-SNE defines a two- or three-dimensional data point such that similar high-dimensional data points are represented by nearby low-dimensional data points with high probability.
It uses Euclidean distance between each pair of high-dimensional data points as the metric for measuring similarity.
However, since our purpose is to construct the low-dimensional representation of the safe region $\mathcal{S}_y$, we are more interested in safety rather than just distance.
Accordingly, we propose a new metric that considers similarity and safety at the same time.

The general motivation for determining the simplified state $y$ is that the safe and unsafe original system states $x$ should be separated in the simplified state space $\mathcal{Y}$.
Since, in this paper, the safe region is defined with respect to the ROA, the trajectories of safe initial states will converge to the origin, while unsafe initial states will have divergent trajectories.
Hence, if two original system states $x$ have similar trajectories under the corrective controller $K(x)$, then ideally they should also have nearby corresponding simplified states $y$ (see Fig.~\ref{fig.tSNE}).
Based on this, we first calculate the pairwise trajectory distance $\omega_{ij}$ between the $i$-th and $j$-th training data, using dynamic time warping (DTW) as
\begin{equation}
    \omega_{ij} = \mathrm{dtw}(\Phi_{\mathrm{sim}} (t;x_{\mathrm{sim}}^i), \Phi_{\mathrm{sim}} (t;x_{\mathrm{sim}}^j)) 
\label{eq.similarity_trajectory}
\end{equation}
where $\mathrm{dtw}(\cdot)$ represents the DTW measurement. 
We thus have $\omega_{ij} = 0$ if $i=j$, and the more similar the trajectories are, the smaller the value of $\omega_{ij}$ is. 

\begin{remark}
Besides DTW, other trajectory distance measures, e.g., Fr{\'e}chet distance, can also be used in~(\ref{eq.similarity_trajectory}).
Changing the distance metric does not affect the applicability of the proposed approach.
However, DTW turns out to be a more suitable metric for trajectories of the dynamical systems we investigated.
\label{remark.distance_metric}
\end{remark}

\begin{figure}[!t]
\centering
\includegraphics[width=0.8\linewidth]{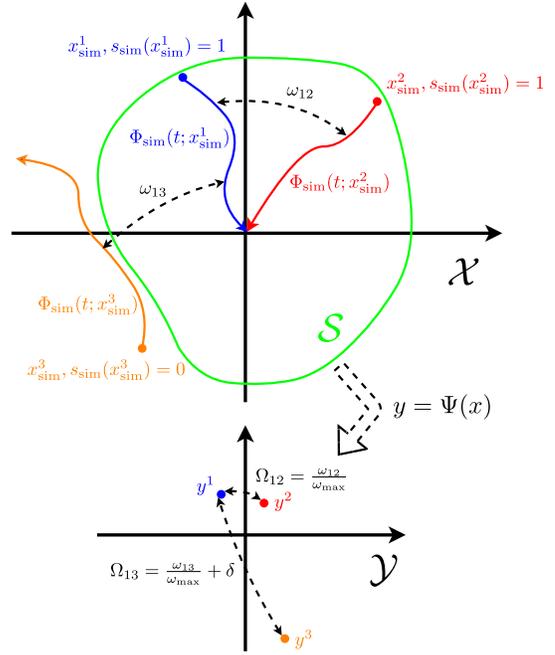}
\caption{The distances $\Omega_{12}$ and $\Omega_{13}$ are computed for three training data $D_{\mathrm{train}}^1$, $D_{\mathrm{train}}^2$, $D_{\mathrm{train}}^3$ using the trajectory distances $\omega_{12}$, $\omega_{13}$ and the safety labels $s_{\mathrm{sim}}(x_{\mathrm{sim}}^1)$, $s_{\mathrm{sim}}(x_{\mathrm{sim}}^2)$, $s_{\mathrm{sim}}(x_{\mathrm{sim}}^3)$. 
Based on these distances, t-SNE calculates the values of corresponding simplified states $y$, where similar and dissimilar training data are modeled by nearby and distant simplified states, respectively.}
\label{fig.tSNE}
\end{figure}

While, in general, the trajectory distance $\omega_{ij}$ reflects the probability that original system states $x_{\mathrm{sim}}^i$ and $x_{\mathrm{sim}}^j$ have the same safety property, it is still possible that safe and unsafe states have similar trajectories. 
To obtain a better low-dimensional safety feature, we thus modify the trajectory distance $\omega_{ij}$ in relation to the safety label $s_{\mathrm{sim}}(x_{\mathrm{sim}})$ and compute the distance $\Omega_{ij}$ between the $i$-th and $j$-th training data as
\begin{equation}
    \Omega_{ij} = \begin{cases} 
    {\displaystyle \frac{\omega_{ij}}{\omega_{\mathrm{max}}} } + \delta, & \text{if } s_{\mathrm{sim}}(x_{\mathrm{sim}}^i) \neq s_{\mathrm{sim}}(x_{\mathrm{sim}}^j) \vspace{2mm} \\ 
   {\displaystyle  \frac{\omega_{ij}}{\omega_{\mathrm{max}}} }, & \text{if } s_{\mathrm{sim}}(x_{\mathrm{sim}}^i) = s_{\mathrm{sim}}(x_{\mathrm{sim}}^j)
\end{cases}
\label{eq.similarity_modified}
\end{equation}
where $\delta$ is a constant and $\omega_{\mathrm{max}} = \max_{i,j} \omega_{ij}$ is the maximum trajectory distance within the training dataset $\mathcal{D}_{\mathrm{train}}$.
The distance $\Omega_{ij}$ is then used as the new metric for t-SNE to measure the similarities between different training data.

In our experiments, we find that a small value of $\delta$ is sufficient for providing a satisfying result of t-SNE (in this paper, for example, we use $\delta = 0.01$).
A large value of $\delta$, in contrast, may lead to information contained in trajectories being ignored, which can reduce the representation power of the learned simplified states $y$.
A sensitivity analysis of the parameter $\delta$ is provided in the supplementary material.

After computing the distance $\Omega_{ij}$ between each pair of training data, we apply t-SNE on the training dataset $\mathcal{D}_{\mathrm{train}}$ to derive a realization of the low-dimensional safety feature. 
To do this, we modify the conditional probability $p_{j|i}$ of t-SNE~\cite{maaten2008visualizing} using the distance $\Omega_{ij}$ as
\begin{equation}
    p_{j|i} = \frac{\mathrm{exp}(-\Omega_{ij}^2 / 2\sigma_i^2)}{ \displaystyle \sum_{k\neq i}\mathrm{exp}(-\Omega_{ik}^2/2\sigma_i^2)}
\label{eq.tsne_p} 
\end{equation}
where $\sigma_i$ is the variance of the Gaussian distribution that is centered on the state $x_{\mathrm{sim}}^i$. 
The remaining computations are the same as in t-SNE.
Since this part makes no contribution, the main steps involved in performing t-SNE are given only in the supplementary material.  
More details are available in~\cite{maaten2008visualizing}.

Using t-SNE, we obtain the values of simplified states $y^1,\ldots, y^{k_t}$ that correspond to the training dataset $\mathcal{D}_{\mathrm{train}}$ as an initial realization of the low-dimensional safety feature.
Such a realization models similar training data with nearby simplified states, e.g., $y^1$ and $y^2$ in Fig.~\ref{fig.tSNE}, and dissimilar training data with distant simplified states, e.g., $y^1$ and $y^3$ in Fig.~\ref{fig.tSNE}.
In general, the simplified state $y$ is chosen to be two- or three-dimensional, i.e., $y \in \mathbb{R}^{n_y}$ with $n_y = 2$ or $n_y =3$.
In this paper, we set $n_y = 2$.

Note that t-SNE only determines the values of simplified states but gives no expression of the state mapping $y = \Psi(x)$.
Therefore, to identify the state mapping $y = \Psi(x)$, we learn a function approximator using the values of simplified states $y^1,\ldots, y^{k_t}$ obtained from t-SNE and the original system states $x_{\mathrm{sim}}^1, \ldots, x_{\mathrm{sim}}^{k_t}$ contained in the training dataset $\mathcal{D}_{\mathrm{train}}$.
This function approximator, e.g., we use a neural network in this paper, is then utilized to represent the state mapping $y = \Psi(x) = \mathrm{NN}(x)$.

\begin{remark}
\label{remark.function_approximator}
Different forms of function approximator, for instance, a Gaussian process, can be used to describe the state mapping $y = \Psi(x)$.
The selection of function approximator depends mainly on the available number of training data.
\end{remark}

Due to the approximation error in the function approximator, some original system states $x$ may have slightly different values in their simplified states $y$ when comparing the initial realization obtained from t-SNE with the one computed from the learned state mapping $y = \Psi(x)$ (for an example, see the simulations in Section~\ref{sec.results_identify} and in particular Fig.~\ref{fig.results_learning}).
Hence, to reduce the influence of this issue on deriving the low-dimensional representation of the safe region $\mathcal{S}_y$, we compute the values of simplified states $y^1,\ldots, y^{k_t}$ once again with the learned state mapping.
This final realization of the low-dimensional safety feature is then used for formulating the SAF $\Gamma(y)$.

\subsection{Belief Function Theory and DSAF}
\label{sec.formulate_mapping}

\begin{figure}[!t]
\centering
\includegraphics[width=0.8\linewidth]{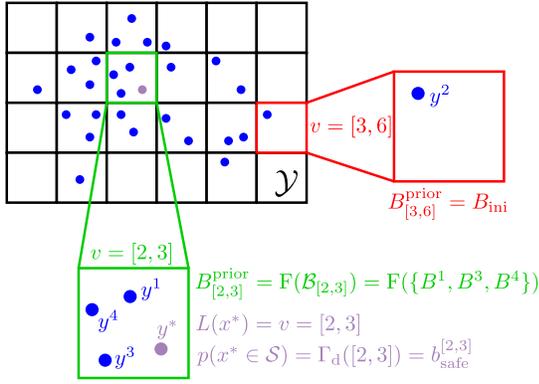}
\caption{The simplified state space $\mathcal{Y}$ is discretized into grid cells. The location of each grid cell is indicated by the index vector $v$. 
The safety of a new original system state, e.g. $x^*$, is estimated by way of the corresponding belief mass as $p(x^* \in \mathcal{S}) = \Gamma_{\mathrm{d}}([2,3]) = b_\mathrm{safe}^{[2,3]}$, where $L(x^*) = v = [2,3]$.
The prior estimate $B_v^{\mathrm{prior}}$ of an index vector $v$ is either obtained by fusing all BBAs within the set $\mathcal{B}_v$, e.g., $B_{[2,3]}^{\mathrm{prior}} = \mathrm{F}(\mathcal{B}_{[2,3]})$, or set to an initial estimate, e.g., $B_{[3,6]}^{\mathrm{prior}} = B_{\mathrm{ini}}$.}
\label{fig.wbf}
\end{figure}

Once the state mapping $y = \Psi(x)$ is determined, we are able to generate the SAF $\Gamma(y)$ using the training dataset $\mathcal{D}_{\mathrm{train}}$.
However, due to the limited size of the training data, it is difficult to construct the SAF $\Gamma(y)$ over the continuous simplified state space $\mathcal{Y}$.
Therefore, we discretize the simplified state space $\mathcal{Y}$.

The range of the simplified state space $\mathcal{Y}$ is determined by the maximum and minimum values of the simplified states $y^1,\ldots, y^{k_t}$ in each dimension.
We then discretize the simplified state space $\mathcal{Y}$ into grid cells with a predefined step size. 
Each grid cell is assigned an index vector $v \in \mathbb{Z}_+^{2}$ to indicate its position in the simplified state space $\mathcal{Y}$; for example, $v = [2,3]$ refers to the grid cell that is located at the second row and third column (see Fig.~\ref{fig.wbf}).
A locating function is defined as follows.

\begin{definition}
\label{def.locating_functioon}
By locating the simplified state $y = \Psi(x)$ for an original system state $x$ in the simplified state space $\mathcal{Y}$, the locating function $L(x)$ returns the index vector $v$ of the grid cell that it belongs to.  
\end{definition}

By assuming that the SAF $\Gamma(y)$ is constant in each grid cell, we obtain a DSAF $\Gamma_{\mathrm{d}}(v)$ that we will have to define. 
Then, instead of using the simplified state $y$, the safety of an original system state $x$ is estimated by way of the index vector $v$ as
\begin{equation}
    p(x \in \mathcal{S}) = \Gamma_{\mathrm{d}}(v)|_{v = L(x)} \sim [0,1].
\label{eq.safety_gamma_discretized}
\end{equation}

In general, the DSAF $\Gamma_{\mathrm{d}}(v)$ for an index vector $v$ can be approximated by the number of safe and unsafe original system states $x$ that map to the corresponding grid cell, i.e., $L(x) = v$.
However, due to the high-dimensional original system state space, it is, in most cases, infeasible to acquire a sufficient amount of data to derive an accurate estimate. 
To solve this problem, we propose using belief function theory~\cite{shafer1976mathematical} to describe the DSAF $\Gamma_{\mathrm{d}}(v)$, where the uncertainty caused by insufficiency in the data amount is considered by a subjective probability~\cite{josang2016subjective}.

Belief function theory is a general approach to modeling epistemic uncertainty that uses a belief mass to represent the probability of the occurrence of an event.
The assignment of belief masses to all possible events is denoted as the basic belief assignment (BBA).
The belief mass on the entire event domain, i.e., the probability that one arbitrary event happens, indicates the subjective uncertainty of the estimate~\cite{josang2016subjective}.
According to this, we define a BBA $B_v$ separately for each index vector $v$ as follows.

\begin{definition}
\label{def.bba_v}
The BBA $B_v$ for an index vector $v$ is given as
\begin{equation}
B_v = (b_{\mathrm{safe}}^v, b_{\mathrm{unsafe}}^v, \mu^v)
\label{eq.bba_v}
\end{equation}
which represents the belief about the value of the DSAF $\Gamma_{\mathrm{d}}(v)$ for the index vector $v$.
The belief masses $b_{\mathrm{safe}}^v$ and $b_{\mathrm{unsafe}}^v$ are the probabilities of the occurrence of two complementary events, i.e., $p(x \in \mathcal{S})$ and $p(x \notin \mathcal{S})$, where the original system state $x$ has the index vector $v$ from the locating function $L(x)$.
$\mu^v$ is the subjective uncertainty that reflects the confidence level of estimating the safety.
$\mu^v = 0$ means we believe that the estimate is absolutely correct.
It holds that
\begin{equation}
b_{\mathrm{safe}}^v + b_{\mathrm{unsafe}}^v + \mu^v = 1
\label{eq.bba_requirement}    
\end{equation}
and $b_{\mathrm{safe}}^v$, $b_{\mathrm{unsafe}}^v$, $\mu^v$ all lie within the interval $[0,1]$.
\end{definition}

Hence the DSAF $\Gamma_{\mathrm{d}}(v)$ is given by the belief masses $b_\mathrm{safe}^{v}$ of the corresponding BBAs $B_v$ as 
\begin{equation}
\Gamma_{\mathrm{d}}(v) = b_\mathrm{safe}^{v}.
\label{eq.safety_estimate}
\end{equation}
The low-dimensional representation of the safe region $\mathcal{S}_y$ is then defined among the discretized simplified state space as
\begin{equation}
\mathcal{S}_y = \{ v \hspace{1mm} | \hspace{1mm} \Gamma_{\mathrm{d}}(v) = b_\mathrm{safe}^{v} > p_t \}
\label{eq.map_m}
\end{equation}
where $p_t$ is the predefined probability threshold.
In the next subsection, we explain how to initialize the DSAF $\Gamma_{\mathrm{d}}(v)$ so as to enable the application of the SRL framework on the real system.

\subsection{Initializing the DSAF from Training Data}
\label{sec.initial_mapping}

Since each training data provides information on the value of the DSAF $\Gamma_{\mathrm{d}}(v)$, the low-dimensional representation of the safe region $\mathcal{S}_y$ is initialized using the training dataset $\mathcal{D}_{\mathrm{train}}$.
By considering each training data as a belief source, we formulate the following BBAs for all training data and later fuse them to derive an initial estimate of the DSAF $\Gamma_{\mathrm{d}}(v)$.

\begin{definition}
\label{def.bba_i}
The BBA $B^i$ obtained from the $i$-th training data $D_{\mathrm{train}}^i$ is defined as
\begin{equation}
    B^i = (b_{\mathrm{safe}}^i, b_{\mathrm{unsafe}}^i, \mu^i).
\label{eq.bba}
\end{equation}
It represents the belief about the value of the DSAF $\Gamma_{\mathrm{d}}(v)$ for the index vector $v= L(x_{\mathrm{sim}}^i)$, where the belief source is the $i$-th training data.
$b_{\mathrm{safe}}^i$, $b_{\mathrm{unsafe}}^i$ and $\mu^i$ have the same meanings as in Definition~\ref{def.bba_v}.
\end{definition}

Due to the inevitable simulation-to-reality gap, we initialize the BBA of each training data with a constant uncertainty $\mu_{\mathrm{ini}} >0$ as
\begin{equation}
B^i = \begin{cases} 
   (1 - \mu_{\mathrm{ini}}, 0, \mu_{\mathrm{ini}}), & \text{if } s_{\mathrm{sim}}(x_{\mathrm{sim}}^i) = 1 \\
   (0, 1-\mu_{\mathrm{ini}}, \mu_{\mathrm{ini}}), & \text{if } s_{\mathrm{sim}}(x_{\mathrm{sim}}^i) = 0
\end{cases}
\label{eq.bba_ini}
\end{equation}
where $i = 1,\ldots,k_t$.
Since no information about the unknown term $d(x)$ is available prior to the learning process on the real system, the initial subjective uncertainties are chosen to be the same for all BBAs.
Later in the online adaptation method, the subjective uncertainties are updated by using the feedback data to realize more accurate safety estimates.

For each index vector $v$, the BBA $B_v$ is then estimated by using the BBAs of the training data.
To achieve this, we first generate a set of BBAs $\mathcal{B}_v$ for each index vector $v$
\begin{equation}
    \mathcal{B}_v = \{B^i \hspace{1mm}|\hspace{1mm} L(x_{\mathrm{sim}}^i) =v \}.
\label{eq.dataset_bv}
\end{equation}
which contains the BBAs of the training data whose original system state $x_{\mathrm{sim}}$ corresponds to the index vector $v$. 
The size of the set $\mathcal{B}_v$ is denoted as $k_v$.

Every BBA in the set $\mathcal{B}_v$ provides a belief about the value of the DSAF $\Gamma_{\mathrm{d}}(v)$ for the index vector $v$.
Hence, an estimate of the BBA $B_v$ is derived by fusing all BBAs within the set $\mathcal{B}_v$ as
\begin{equation}
B_v^{\mathrm{prior}} = (b_{\mathrm{safe}}^{v,\mathrm{prior}}, b_{\mathrm{unsafe}}^{v,\mathrm{prior}}, \mu^{v,\mathrm{prior}}) = \begin{cases} 
   \mathrm{F}(\mathcal{B}_v), & \text{if } k_v \geq k_{\mathrm{min}} \\
   B_{\mathrm{ini}}, & \text{else }
\end{cases}
\label{eq.bba_fusion}
\end{equation}
where $B_{\mathrm{ini}}$ is an initial estimate that represents our guess about the BBA $B_v$ when no training data is available (see Fig.~\ref{fig.wbf}).
$\mathrm{F}(\cdot)$ is a fusion operation among the set $\mathcal{B}_v$, which is referred to as weighted belief fusion and is defined according to~\cite{josang2018categories} as 
\begin{equation}
    b_\mathrm{safe}^{v,\mathrm{prior}} = \frac{ {\displaystyle \sum_{B^i \in \mathcal{B}_v} } b_{\mathrm{safe}}^i(1 - \mu^i) {\displaystyle  \prod_{ \substack {B^j \in \mathcal{B}_v \\  i \neq j} }} \mu^j }{\displaystyle \left(\sum_{B^i \in \mathcal{B}_v} \prod_{ \substack {B^j \in \mathcal{B}_v \\  i \neq j} } \mu^j \right) - k_v \prod_{B^i \in \mathcal{B}_v} \mu^i } 
\label{eq.wbf_safe}
\end{equation}
\begin{equation}
    b_\mathrm{unsafe}^{v,\mathrm{prior}} = \frac{ {\displaystyle \sum_{B^i \in \mathcal{B}_v} } b_{\mathrm{unsafe}}^i(1 - \mu^i) {\displaystyle  \prod_{ \substack {B^j \in \mathcal{B}_v \\  i \neq j} }} \mu^j }{\displaystyle \left(\sum_{B^i \in \mathcal{B}_v} \prod_{ \substack {B^j \in \mathcal{B}_v \\  i \neq j} } \mu^j \right) - k_v \prod_{B^i \in \mathcal{B}_v} \mu^i } 
\label{eq.wbf_unsafe} 
\end{equation}
\begin{equation}
    \mu^{v,\mathrm{prior}} = \frac{\displaystyle  \left( k_v - \sum_{B^i \in \mathcal{B}_v} \mu^i \right) \prod_{B^i \in \mathcal{B}_v} \mu^i} {\displaystyle \left(\sum_{B^i \in \mathcal{B}_v} \prod_{ \substack {B^j \in \mathcal{B}_v \\  i \neq j} } \mu^j \right) - k_v \prod_{B^i \in \mathcal{B}_v} \mu^i }. 
\label{eq.wbf_uncertainty}
\end{equation}
We refer to this estimate of the BBA $B_v$ as the prior estimate $B_v^{\mathrm{prior}}$. 
Since it is still likely to be imprecise if the available number of training data is too small, the fusion is performed only when the number of BBAs contained in the set $\mathcal{B}_v$ is not smaller than a minimum number $k_{\mathrm{min}}$.
Otherwise, the prior estimate $B_v^{\mathrm{prior}}$ is set to the initial estimate $B_{\mathrm{ini}}$.  
We use $B_{\mathrm{ini}} = (0.05,0.55,0.4)$ in our experiments. 
This means that if there is very little experience available in the form of training data for one grid cell, then the respective states will initially be considered as unsafe.
The resulting prior estimate $B_v^{\mathrm{prior}}$ is a BBA that satisfies 
\begin{equation}
    b_\mathrm{safe}^{v,\mathrm{prior}} + b_\mathrm{unsafe}^{v,\mathrm{prior}} + \mu^{v,\mathrm{prior}} = 1
\end{equation}
and $b_\mathrm{safe}^{v,\mathrm{prior}}$, $b_\mathrm{unsafe}^{v,\mathrm{prior}}$, $\mu^{v,\mathrm{prior}}$ all lie within the interval $[0,1]$.

After computing the prior estimate $B_v^{\mathrm{prior}}$ for all index vectors $v$, we thus obtain a prior DSAF $\Gamma_{\mathrm{d}}^{\mathrm{prior}}(v)$ 
\begin{equation}
\Gamma_{\mathrm{d}}^{\mathrm{prior}}(v) = b_\mathrm{safe}^{v,\mathrm{prior}}
\label{eq.prior_mapping}
\end{equation}
which delivers an estimate of the DSAF $\Gamma_{\mathrm{d}}(v)$ that is derived from the training data.
The low-dimensional representation of the safe region $\mathcal{S}_y$ is then initialized by letting $\Gamma_{\mathrm{d}}(v) = \Gamma_{\mathrm{d}}^{\mathrm{prior}}(v)$.
In the next section, we propose an online adaptation method to update the DSAF $\Gamma_{\mathrm{d}}(v)$ using feedback data, to account for the unknown part of the system dynamics $d(x)$.

\section{Online Adaptation of the Safety Assessment Function}
\label{sec.online_adaptation}
In the early learning phase with the real system, the prior DSAF $\Gamma_{\mathrm{d}}^{\mathrm{prior}}(v)$ allows a rough estimate of the safety of an original system state.
During the learning process, the feedback data is used to update the DSAF $\Gamma_{\mathrm{d}}(v)$ to achieve more accurate safety estimates.
Each update iteration of the DSAF $\Gamma_{\mathrm{d}}(v)$ consists of three steps.
First, we modify the prior DSAF $\Gamma_{\mathrm{d}}^{\mathrm{prior}}(v)$ by revising the subjective uncertainties of the BBAs of the training data.
Second, we compute a feedback DSAF $\Gamma_{\mathrm{d}}^{\mathrm{feedback}}(v)$ using the feedback data.
Third, the updated DSAF $\Gamma_{\mathrm{d}}(v)$ is obtained by fusing the prior and feedback DSAFs. 
Note that each time the corrective controller $K(x)$ is activated for the real system, we obtain new feedback data.
Hence the size of the feedback dataset $\mathcal{D}_{\mathrm{feedback}}$ increases incrementally during the learning process.
For simplicity, we consider the feedback dataset $\mathcal{D}_{\mathrm{feedback}}$ of size $k_f$ in this section.
Details of the online adaptation method are given in the following.

\subsection{Update of the Prior DSAF with Feedback Data}
\label{sec.update_prior}
The prior DSAF $\Gamma_{\mathrm{d}}^{\mathrm{prior}}(v)$ is constructed using the training dataset $\mathcal{D}_{\mathrm{train}}$, in which the uncertainty caused by the unknown term $d(x)$ is represented by the subjective uncertainty $\mu^i$ of each BBA $B^i$.
Hence, the update of the prior DSAF $\Gamma_{\mathrm{d}}^{\mathrm{prior}}(v)$ will now modify the subjective uncertainties by accounting for new information given by feedback data.
For this, we assume that original system states that are in close proximity to each other most probably have similar safety properties.

\begin{assumption}
The probability $p(s_{\mathrm{real}}(x^1) = s_{\mathrm{real}}(x^2))$ that two original system states $x^1$ and $x^2$ have the same safety property on the real system is inversely proportional to their Euclidean distance in the original state space $||x^1 - x^2||$.
\label{ass.safety}
\end{assumption}

In addition, we define a function $P(x)$ to quantify the similarity with respect to the safety of nominal and real system trajectories that start in the same initial original system state $x$
\begin{equation}
    P(x) = p( s_{\mathrm{sim}}(x) = s_{\mathrm{real}}(x) ) \sim [0,1].
\end{equation}
It represents the probability that for a given original system state $x$, its safety label $s_{\mathrm{sim}}(x)$ obtained with the nominal system is the same as the safety label $s_{\mathrm{real}}(x)$ obtained with the real system.
Then, according to Assumption~\ref{ass.safety}, if we observe an original system state $x$ that has the same safety property both in simulation and in reality, it is likely that other original system states that are close to the observed state will also show the same safety property. 

In order to predict the value of the function $P(x)$, we approximate it with a Gaussian process regression (GPR) model $P(x) = \mathrm{GP}(x)$.
For each original system state $x_\mathrm{real}$ contained in the feedback dataset $\mathcal{D}_{\mathrm{feedback}}$, we examine its safety label $s_{\mathrm{sim}}(x_{\mathrm{real}})$ in simulation.
This leads to a set of samples $\{P(x_{\mathrm{real}}^1),\ldots,P(x_{\mathrm{real}}^{k_f}) \}$ for the function $P(x)$, in which
\begin{equation}
P(x_{\mathrm{real}}^i) = \begin{cases} 
  1, & \text{if } s_{\mathrm{sim}}(x_{\mathrm{real}}^i) = s_{\mathrm{real}}(x_{\mathrm{real}}^i) \\
  0, & \text{if } s_{\mathrm{sim}}(x_{\mathrm{real}}^i) \neq s_{\mathrm{real}}(x_{\mathrm{real}}^i)
\end{cases}
\label{eq.sample_p}
\end{equation}
for $i = 1,\ldots,k_f$.
Hence the GPR model $\mathrm{GP}(x)$ is trained with the sets $\{x_{\mathrm{real}}^1, \ldots, x_{\mathrm{real}}^{k_f} \}$ and $\{P(x_{\mathrm{real}}^1),\ldots,P(x_{\mathrm{real}}^{k_f}) \}$, which are obtained from the current feedback dataset $\mathcal{D}_{\mathrm{feedback}}$. 

\begin{remark}
\label{remark.testing_in_reality}
If the real system is a real-world dynamical system, then it is usually difficult to test the corrective controller $K(x)$ with arbitrary initial original system states $x$ in reality, since there is a high risk of encountering unsafe behaviors. 
However in contrast, the simulation can be initialized with any original system state $x_\mathrm{real}$ contained in the feedback data, which then makes it possible to approximate the function $P(x)$.
\end{remark}

\begin{figure}[!t]
\centering
\includegraphics[width=0.7\linewidth]{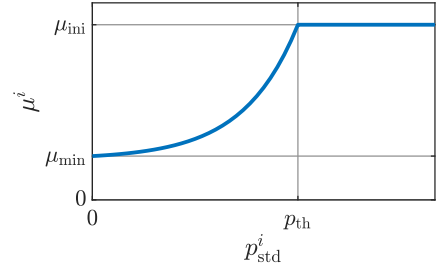}
\caption{
As given in \eqref{eq.bba_update_mu}, the subjective uncertainty $\mu^i$ in the BBA $B^i$ of the $i$-th training data is determined using the corresponding standard deviation $p_{\mathrm{std}}^i$ obtained from the GPR model $\mathrm{GP}(x)$.
}
\label{fig.update_mu}
\end{figure}

The trained GPR model $\mathrm{GP}(x)$ is then used to update the BBA $B^i$ of each training data.
The general motivation is that, we decrease the subjective uncertainty $\mu^i$ if we are confident about the reliability of this training data. 
Hence for the $i$-th training data, we compute a predicted mean value of the function $P(x_{\mathrm{sim}}^i)$, denoted as $p_{\mathrm{mean}}^i$, from the GPR model $\mathrm{GP}(x)$, along with a corresponding standard deviation $p_{\mathrm{std}}^i$ of the predicted value.  
Since a low value of the standard deviation $p_{\mathrm{std}}^i$ means we have observed enough feedback data to make a reliable prediction, we only update the BBA $B^i$ if the standard deviation $p_{\mathrm{std}}^i$ is smaller than a predefined threshold $p_{\mathrm{th}}$ 
\begin{equation}
B^i = \begin{cases} 
   (p_{\mathrm{mean}}^i(1- \mu^i), (1- p_{\mathrm{mean}}^i)(1-\mu^i), \mu^i), \\ \hspace{15mm} \text{if }p_{\mathrm{std}}^i \leq p_{\mathrm{th}} \text{ and } s_{\mathrm{sim}}(x_{\mathrm{sim}}^i) = 1 \\
   ((1- p_{\mathrm{mean}}^i)(1- \mu^i), p_{\mathrm{mean}}^i(1-\mu^i), \mu^i), \\ \hspace{15mm} \text{if }p_{\mathrm{std}}^i \leq p_{\mathrm{th}} \text{ and } s_{\mathrm{sim}}(x_{\mathrm{sim}}^i) = 0
\end{cases}
\label{eq.bba_update}
\end{equation}
with the new subjective uncertainty $\mu^i$ calculated as
\begin{equation}
\mu^i =   \frac{\mu_{\mathrm{ini}} - \mu_{\mathrm{min}}}{\alpha^{p_{\mathrm{th}}} - 1} (\alpha^{p_{\mathrm{std}}^i} - 1) + \mu_{\mathrm{min}} 
\label{eq.bba_update_mu}
\end{equation}
where $\mu_{\mathrm{ini}}$ is the same initial subjective uncertainty as that given in \eqref{eq.bba_ini} (see Fig.~\ref{fig.update_mu} for a graphical representation of \eqref{eq.bba_update_mu}).
BBAs $B^i$ with $p_{\mathrm{std}}^i > p_{\mathrm{th}}$ remain unchanged, as in (\ref{eq.bba_ini}). 
Such an update of the BBA $B^i$ considers the predicted value of the function $P(x_{\mathrm{sim}}^i)$ and the reliability of this prediction at the same time.

(\ref{eq.bba_update_mu}) is designed by considering two aspects: first, the subjective uncertainty $\mu^i$ is set equal to $\mu_{\mathrm{ini}}$ when $p_{\mathrm{std}}^i \geq p_{\mathrm{th}}$. 
This means that in this case we do not have the confidence to update the BBA $B^i$, as not enough information is observed from the feedback data;
second, due to the inevitable reality gap, the subjective uncertainty $\mu^i$ maintains a minimum uncertainty $\mu_{\mathrm{min}}$ even when the standard deviation $p_{\mathrm{std}}^i$ is $0$.
We use the exponential form such that the decrease in $\mu^i$ is faster when the standard deviation $p_{\mathrm{std}}^i$ is near the threshold $p_{\mathrm{th}}$.
The parameter $\alpha >1$ determines the decay rate and is selected by considering the actual learning task.

Note that for the same training data, the relationship between the standard deviation $p_{\mathrm{std}}^i$ and the threshold $p_{\mathrm{th}}$ can change during the learning process.
For example, we might obtain $p_{\mathrm{std}}^i \leq p_{\mathrm{th}}$ in the current update iteration, but in the next update iteration it changes to $p_{\mathrm{std}}^i > p_{\mathrm{th}}$.
This happens primarily when we first observe a safe original system state but followed by a nearby unsafe state, such that the safety of the states in between these two observed states becomes uncertain.
In such cases, we set the BBA $B^i$ back to the initial BBA given in (\ref{eq.bba_ini}).

Once the BBAs $B^i$ of all training data have been updated with the up-to-date feedback dataset $\mathcal{D}_{\mathrm{feedback}}$, the prior estimate $B_v^{\mathrm{prior}}$ for each index vector $v$ is recomputed using (\ref{eq.bba_fusion}).
This results in an updated prior DSAF $\Gamma_{\mathrm{d}}^{\mathrm{prior}}(v)$, which is used later for revising the DSAF $\Gamma_{\mathrm{d}}(v)$.

\subsection{Feedback DSAF}
\label{sec.update_feedback}
The feedback data contain the information about the real safety properties of different original system states $x$.
To fully utilize this valuable information, we construct an additional DSAF, denoted as the feedback DSAF $\Gamma_{\mathrm{d}}^{\mathrm{feedback}}(v)$, using the feedback dataset $\mathcal{D}_{\mathrm{feedback}}$.

As the amount of data is insufficient, we also consider the estimate obtained from the feedback data as a subjective probability~\cite{zhou2020general}.
Then, as with the prior estimate $B_v^{\mathrm{prior}}$, we formulate another estimate of the BBA $B_v$ for each index vector $v$ as
\begin{equation}
B_v^{\mathrm{feedback}} = (b_{\mathrm{safe}}^{v,\mathrm{feedback}}, b_{\mathrm{unsafe}}^{v,\mathrm{feedback}}, \mu^{v,\mathrm{feedback}})
\label{eq.bba_feedback}
\end{equation}
which is referred to as the feedback estimate $B_v^{\mathrm{feedback}}$.

For each index vector $v$, the feedback estimate $B_v^{\mathrm{feedback}}$ is determined by the number of safe and unsafe feedback data that correspond to this grid cell.
By sorting the feedback dataset $\mathcal{D}_{\mathrm{feedback}}$ with the locating function $L(x)$, we denote the number of safe feedback data that have the index vector $v$ from the locating function, i.e., $L(x_{\mathrm{real}}) = v$ and $s_{\mathrm{real}}(x_{\mathrm{real}}) = 1$, as $k_{\mathrm{safe}}^v$ (and $k_{\mathrm{unsafe}}^v$ for the number of unsafe feedback data).
If at least one feedback data is available for the index vector $v$, i.e., $k_{\mathrm{safe}}^v + k_{\mathrm{unsafe}}^v \geq 1$, we compute the feedback estimate $B_v^{\mathrm{feedback}}$ as follows
\begin{eqnarray}
    && b_{\mathrm{safe}}^{v,\mathrm{feedback}} = \frac{k_{\mathrm{safe}}^v}{k_{\mathrm{safe}}^v + k_{\mathrm{unsafe}}^v} (1 - \mu^{v,\mathrm{feedback}}) 
\label{eq.bba_real_safe} \\
    && b_{\mathrm{unsafe}}^{v,\mathrm{feedback}} = \frac{k_{\mathrm{unsafe}}^v}{k_{\mathrm{safe}}^v + k_{\mathrm{unsafe}}^v}(1 - \mu^{v,\mathrm{feedback}}) 
\label{eq.bba_real_unsafe} \\
    && \mu^{v,\mathrm{feedback}}  = \beta \mathrm{exp}(-\gamma(k_{\mathrm{safe}}^v + k_{\mathrm{unsafe}}^v - 1)).
\label{eq.bba_real_mu}
\end{eqnarray}
The subjective uncertainty $\mu^{v,\mathrm{feedback}}$ decreases if more feedback data are observed for the index vector $v$.
It satisfies that, if a sufficient number of feedback data is obtained, the subjective uncertainty $\mu^{v,\mathrm{feedback}}$ approaches 0.
In such a case, the belief masses $b_{\mathrm{safe}}^{v,\mathrm{feedback}}$ and $b_{\mathrm{unsafe}}^{v,\mathrm{feedback}}$ can be considered as the actual probabilities.
The parameters $\beta$ and $\gamma$ define the initial value and
the decay rate of the subjective uncertainty $\mu^{v,\mathrm{feedback}}$, respectively.
If no feedback data is observed for the index vector $v$, we set the feedback estimate $B_v^{\mathrm{feedback}}$ to an empty BBA $B_{\varnothing}$ defined as $B_v^{\mathrm{feedback}} = B_{\varnothing} = (0,0,1)$, which indicates that no safety estimate can be made.

Using the feedback estimate $B_v^{\mathrm{feedback}}$, we obtain the following feedback DSAF $\Gamma_{\mathrm{d}}^{\mathrm{feedback}}(v)$
\begin{equation}
\Gamma_{\mathrm{d}}^{\mathrm{feedback}}(v) = b_\mathrm{safe}^{v,\mathrm{feedback}}
\label{eq.feedback_mapping}
\end{equation}
which represents the estimate of the DSAF $\Gamma_{\mathrm{d}}(v)$ derived from the feedback data only.
In the next subsection, we fuse the feedback DSAF $\Gamma_{\mathrm{d}}^{\mathrm{feedback}}(v)$ with the updated prior DSAF $\Gamma_{\mathrm{d}}^{\mathrm{prior}}(v)$ to derive a more accurate DSAF $\Gamma_{\mathrm{d}}(v)$.

\subsection{Fusion of Prior and Feedback DSAFs}
\label{sec.update_safety}
The prior and feedback DSAFs both provide beliefs about safety by using different datasets as their belief source.
To update the DSAF $\Gamma_{\mathrm{d}}(v)$, we fuse these two functions using weighted belief fusion as given in (\ref{eq.wbf_safe}-\ref{eq.wbf_uncertainty}).
This leads to a fused estimate $B_v^{\mathrm{fuse}}$ for each index vector $v$
\begin{equation}
B_v^{\mathrm{fuse}} = (b_{\mathrm{safe}}^{v,\mathrm{fuse}}, b_{\mathrm{unsafe}}^{v,\mathrm{fuse}}, \mu^{v,\mathrm{fuse}}) 
\label{eq.bba_combine}
\end{equation}
which is computed as
\begin{equation}
   B_v^{\mathrm{fuse}} = \begin{cases} 
   \mathrm{F}( \{B_v^{\mathrm{prior}}, B_v^{\mathrm{feedback}} \}), & \text{if } B_v^{\mathrm{feedback}} \neq B_{\varnothing} \\
   B_v^{\mathrm{prior}}, & \text{if } B_v^{\mathrm{feedback}} = B_{\varnothing}.
\label{eq.bba_combine_single}
\end{cases}
\end{equation}
If the feedback estimate $B_v^{\mathrm{feedback}}$ is non-empty, we find the fused estimate $B_v^{\mathrm{fuse}}$ through weighted belief fusion $\mathrm{F}(\cdot)$ of the set $\{B_v^{\mathrm{prior}}, B_v^{\mathrm{feedback}} \}$.
Otherwise, we set the fused estimate $B_v^{\mathrm{fuse}}$ equal to the prior estimate $B_v^{\mathrm{prior}}$.

The fused estimate $B_v^{\mathrm{fuse}}$ fulfills the following property, which is also given in~\cite{zhou2020general}.

\newtheorem{proposition}{Proposition}
\begin{proposition}
\label{prop.convergence} 
If the number of feedback data approaches infinity, the fused estimate $B_v^{\mathrm{fuse}}$ becomes the actual probabilities, and the prior estimate $B_v^{\mathrm{prior}}$ has no effect in making safety estimates.
\end{proposition}
\begin{proof}
Proposition~\ref{prop.convergence} is justified by the following equations
\begin{eqnarray}
    && \lim_{k_{\mathrm{safe}}^v + k_{\mathrm{unsafe}}^v \rightarrow \infty} b_{\mathrm{safe}}^{v,\mathrm{fuse}}    = b_{\mathrm{safe}}^{v,\mathrm{feedback}}  \\ 
    && \lim_{k_{\mathrm{safe}}^v + k_{\mathrm{unsafe}}^v \rightarrow \infty} b_{\mathrm{unsafe}}^{v,\mathrm{fuse}} = b_{\mathrm{unsafe}}^{v,\mathrm{feedback}}  \\
    && \lim_{k_{\mathrm{safe}}^v + k_{\mathrm{unsafe}}^v \rightarrow \infty} \mu^{v,\mathrm{fuse}} = \mu^{v,\mathrm{feedback}} =0
    \label{Eq.relationship}
\end{eqnarray}
which are obtained by simplifying (\ref{eq.wbf_safe}-\ref{eq.wbf_uncertainty}) with the set $\{B_v^{\mathrm{prior}}, B_v^{\mathrm{feedback}} \}$.
\end{proof}

Considering computational efficiency, the update of the DSAF $\Gamma_{\mathrm{d}}(v)$ is generally performed once when every $k_u$ feedback data is obtained, where the value of $k_u$ is selected according to the actual learning task. 
In each update iteration (indexed by number $N$, see Section~\ref{sec.results_update}), we first use the up-to-date feedback dataset $\mathcal{D}_{\mathrm{feedback}}$ to update the prior DSAF $\Gamma_{\mathrm{d}}^{\mathrm{prior}}(v)$ and to construct the feedback DSAF $\Gamma_{\mathrm{d}}^{\mathrm{feedback}}(v)$. 
Then, the fused estimate $B_v^{\mathrm{fuse}}$ is computed from these two functions for each index vector $v$.
The updated DSAF $\Gamma_{\mathrm{d}}(v)$ is thus obtained using the fused estimate $B_v^{\mathrm{fuse}}$ as
\begin{equation}
\Gamma_{\mathrm{d}}(v) = b_\mathrm{safe}^{v,\mathrm{fuse}}
\label{eq.safety_mapping}
\end{equation}
which also gives the latest low-dimensional representation of the safe region $\mathcal{S}_y$ according to (\ref{eq.map_m}).
With further feedback data, the DSAF $\Gamma_{\mathrm{d}}(v)$ becomes more accurate and more reliable safety estimates are obtained.

\section{Quadcopter Experiments}
\label{sec.results}
In this section, we demonstrate the proposed approach for identifying the low-dimensional representation of the safe region $\mathcal{S}_y$, using the example of a quadcopter. 

\subsection{Experimental Setup}

\label{sec.results_setup}
\begin{figure}[!t]
\centering
\includegraphics[width=.9\linewidth]{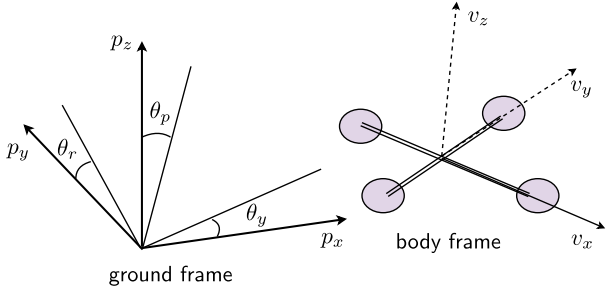}
\caption{The system state $x$ of a quadcopter is defined using the ground frame and the body frame.}
\label{fig.quadcopter}
\end{figure}

We simulate the quadcopter using the system dynamics given in~\cite{luukkonen2011modelling} with MATLAB Simulink\footnote{\url{https://www.mathworks.com/products/simulink.html}} (Version R2019b) running on a laptop powered by an Intel i7-7700HQ CPU.
The 12-dimensional system state is defined as $x = [p_g, \theta_g, v_b, \omega_b]^T$, where $p_g = [p_x, p_y, p_z]^T$ and $\theta_g = [\theta_r, \theta_p, \theta_y]^T$ are the linear and angular positions defined in the ground frame, $v_b = [v_x, v_y, v_z]^T$ and $\omega_b = [\omega_r, \omega_p, \omega_y]^T$ are the linear and angular velocities defined in the body frame (see Fig.~\ref{fig.quadcopter}).
The control input $u$ consists of the four motor speeds of the quadcopter.
For the nominal system model, we set the mass of the quadcopter to $m = 1 \text{ kg}$ and the maximal lifting force to $f = 200 \text{ N}$.
The safety of a given state $x$ is determined by simulating the controlled dynamics with the corrective control $K(x)$ that starts in initial state $x$, and checking if the controller is able to successfully drive the quadcopter back to a hovering state without crashing.
In this example, we use the PID controller given in~\cite{luukkonen2011modelling} as the corrective controller $K(x)$.
It stabilizes the quadcopter's height as well as its roll, pitch and yaw rotations.
The coefficients of the PID controller are: $K_{P,h} = 1.5$, $K_{I,h} = 0$, $K_{D,h} = 2.5$ for the height control, and $K_{P,r} = K_{P,p} = K_{P,y} = 6$, $K_{I,r} = K_{I,p} = K_{I,y} = 0$, $K_{D,r} = K_{D,p} = K_{D,y} = 1.75$ for the roll, pitch and yaw rotations control, respectively.

To generate the training dataset $\mathcal{D}_{\mathrm{train}}$, we first create $k_t = 10000$ original system states $x$. 
We set $p_x = p_y = 0$ and $p_z = 2 \text{ m}$ to leave enough space and time for the corrective controller $K(x)$.
All other variables are sampled with a uniform distribution within the following range:
$ 0 \leq \theta_r, \theta_p, \theta_y \leq 2\pi \text{ rad}$, $-3 \text{ m/s} \leq v_x, v_y, v_z \leq 3 \text{ m/s}$, $-10 \text{ rad/s} \leq \omega_r, \omega_p, \omega_y \leq 10 \text{ rad/s}$.
The training dataset $\mathcal{D}_{\mathrm{train}}$ is then obtained by examining the performance of the corrective controller $K(x)$ for all these initial values.

\subsection{Identifying the Low-dimensional Representation of the Safe Region}
\label{sec.results_identify}

\begin{figure}[!t]
\centering
\subfloat[]{\includegraphics[width=.7\linewidth]{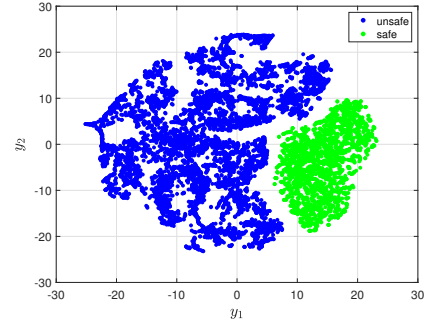}%
\label{fig.results_tSNE}}
\hfil
\subfloat[]{\includegraphics[width=.7\linewidth]{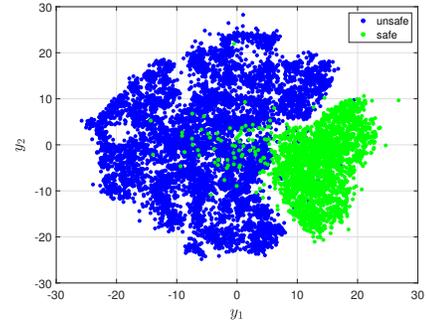}%
\label{fig.results_mapping_NN}}
\caption{(a) The initial realization of simplified states $y^1,\ldots, y^{k_t}$ obtained from t-SNE. The safe and unsafe training data are denoted by green and blue points, respectively. (b) The final realization of simplified states $y^1,\ldots, y^{k_t}$ obtained by recomputing with the learned neural network that represents the state mapping $y = \Psi(x) = \mathrm{NN}(x)$.}
\label{fig.results_learning}
\end{figure}

\begin{figure*}[!t]
\centering
\subfloat[]{\includegraphics[width=1.7in]{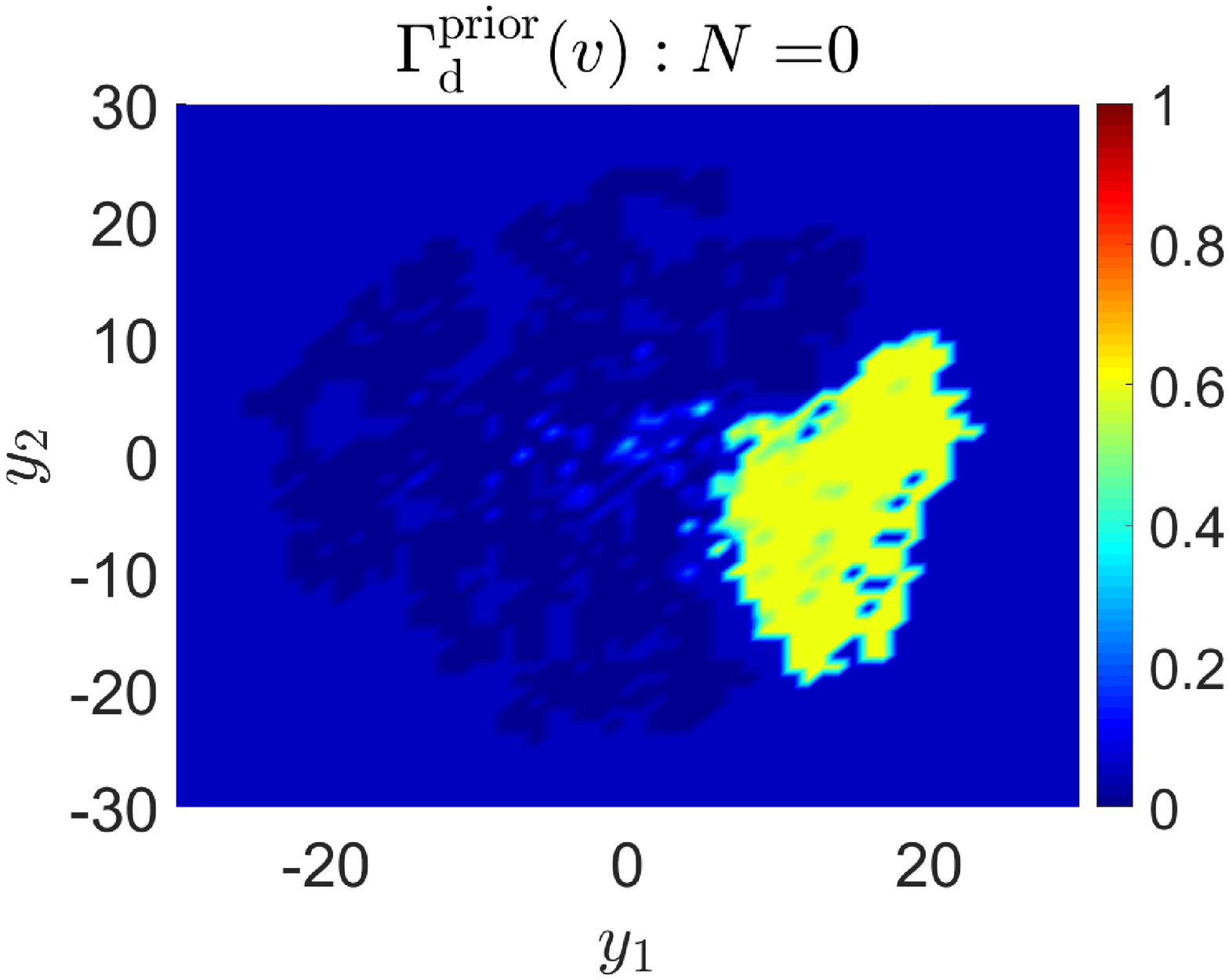}%
\label{fig.sim_0}}
\hfil
\subfloat[]{\includegraphics[width=1.7in]{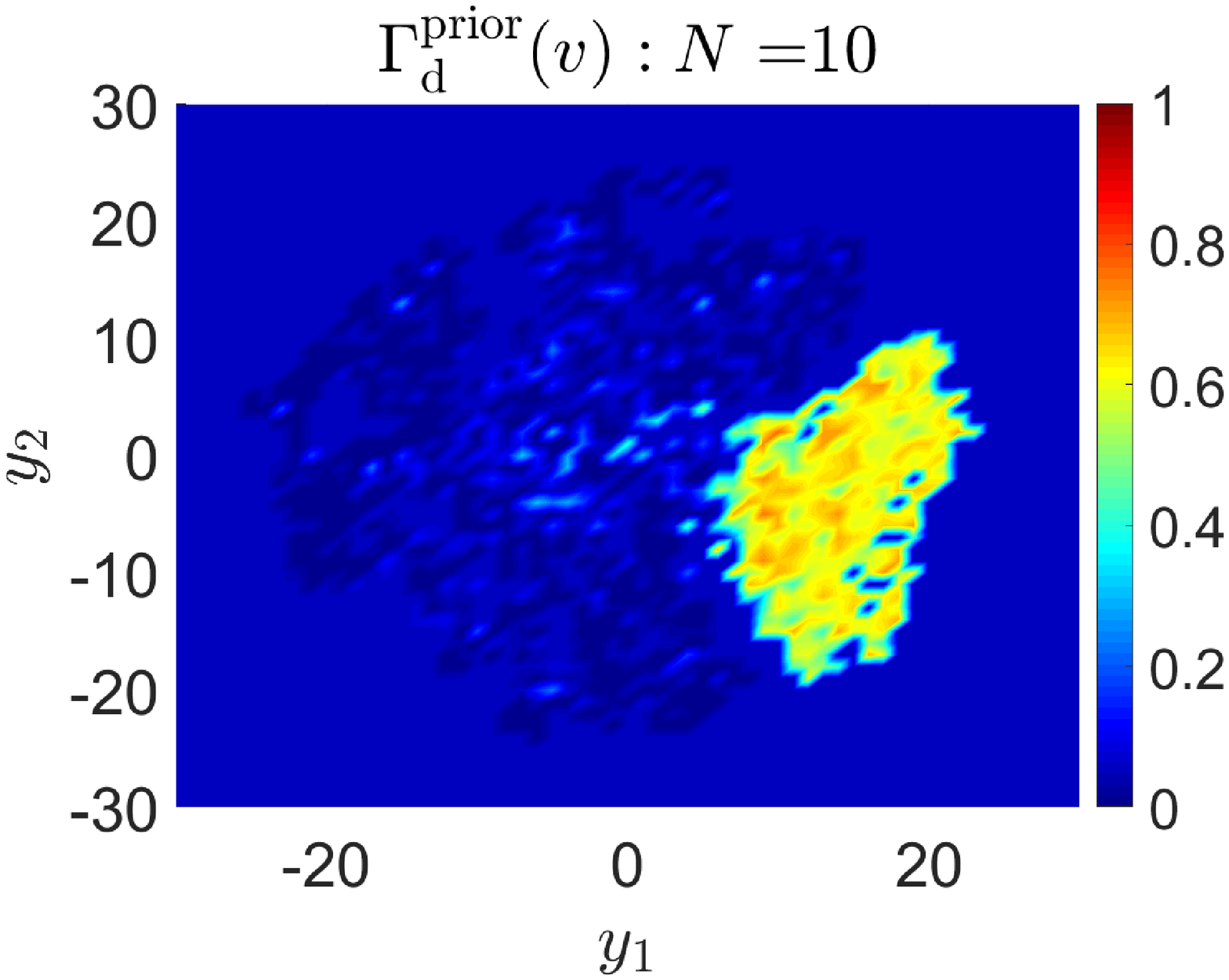}%
\label{fig.sim_10}}
\hfil
\subfloat[]{\includegraphics[width=1.7in]{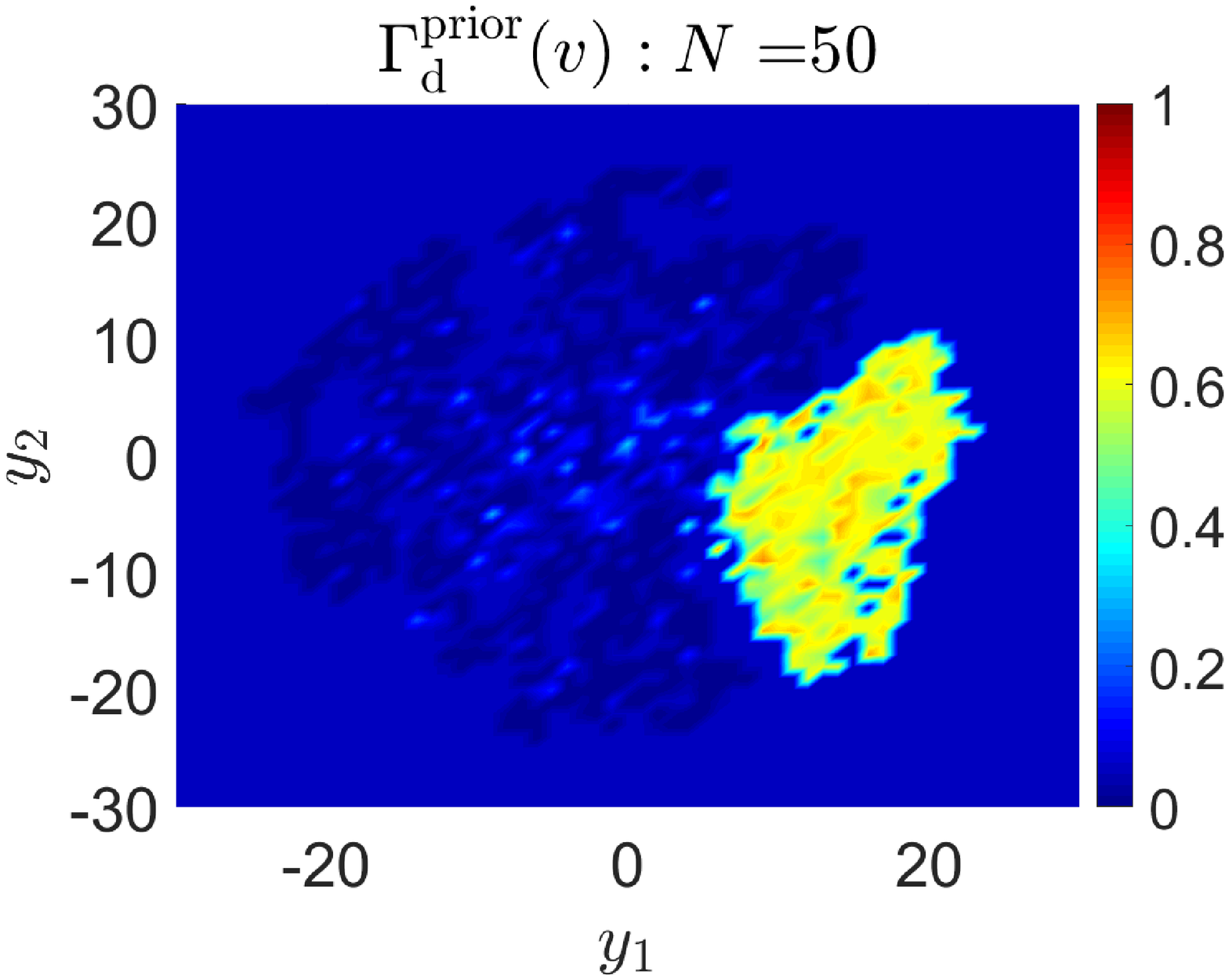}%
\label{fig.sim_50}}
\hfil
\subfloat[]{\includegraphics[width=1.7in]{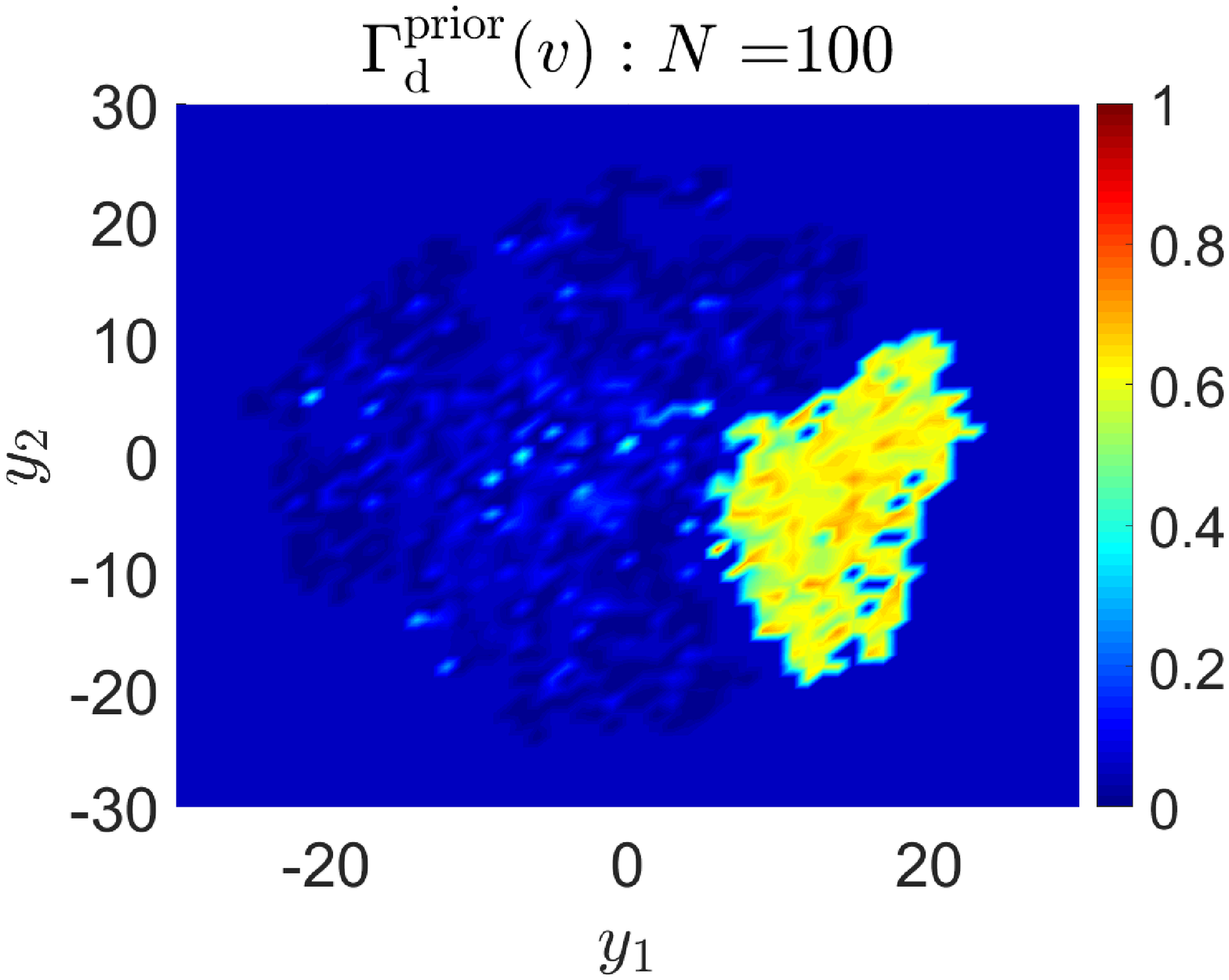}%
\label{fig.sim_100}}
\hfil
\subfloat[]{\includegraphics[width=1.7in]{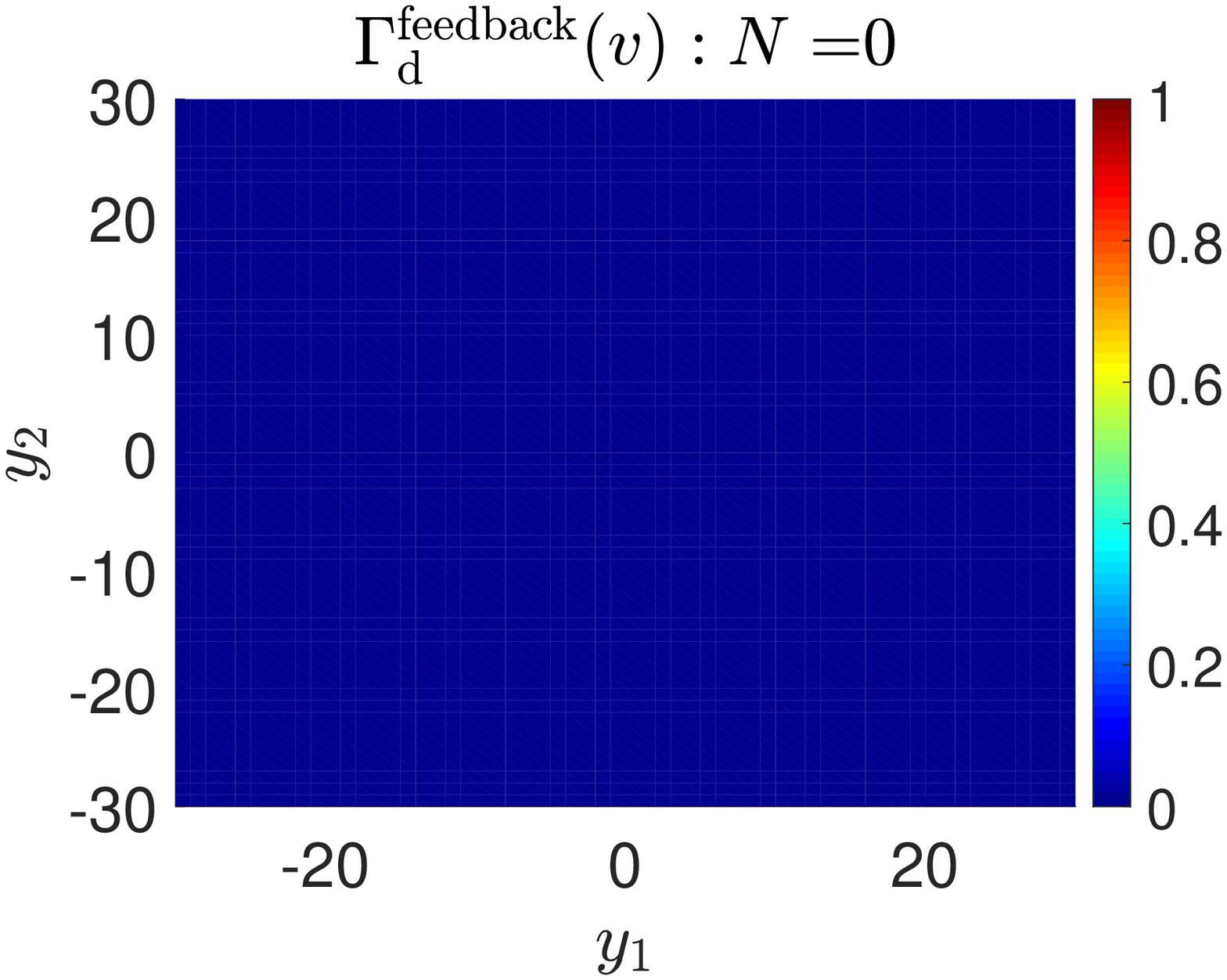}%
\label{fig.real_0}}
\hfil
\subfloat[]{\includegraphics[width=1.7in]{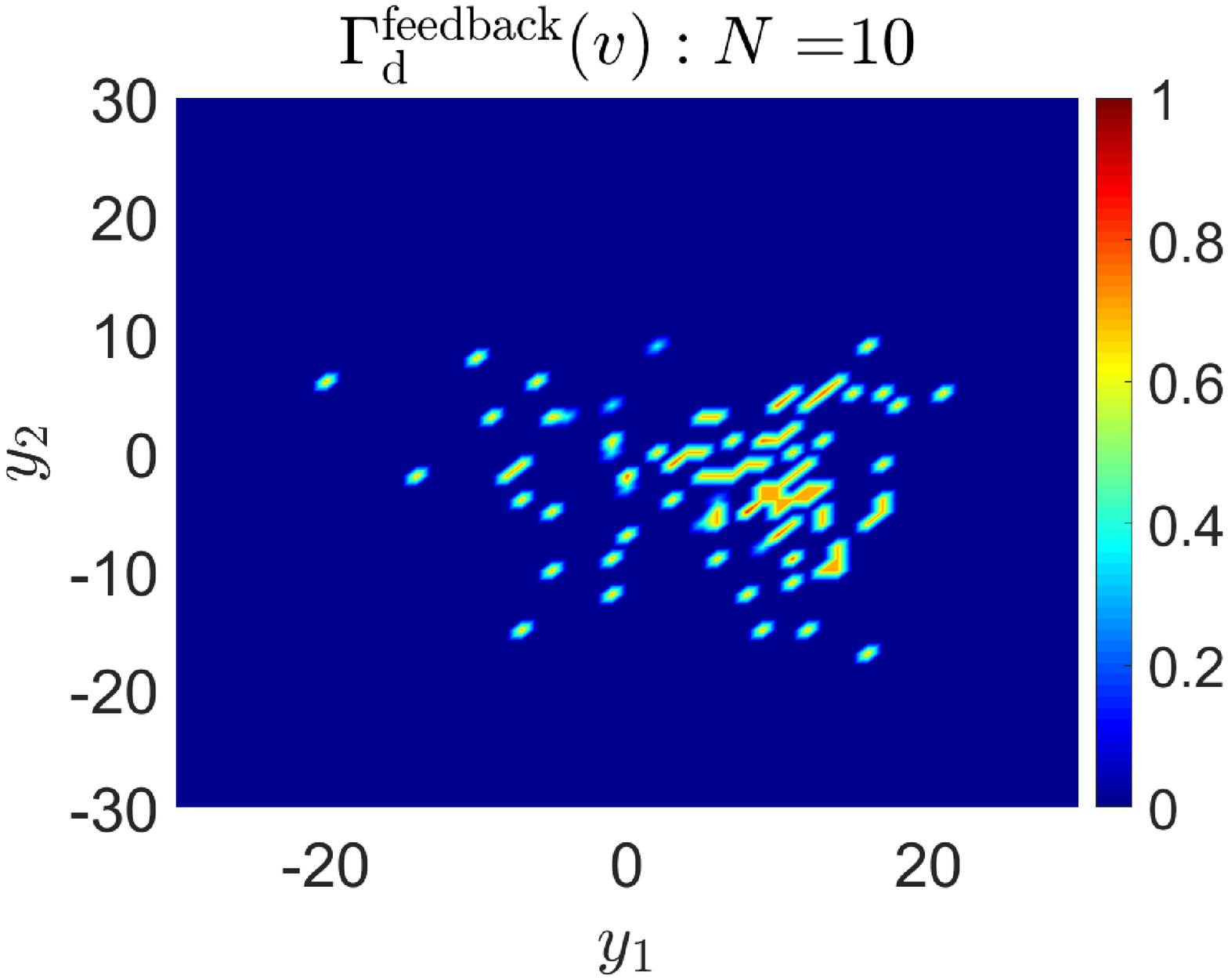}%
\label{fig.real_10}}
\hfil
\subfloat[]{\includegraphics[width=1.7in]{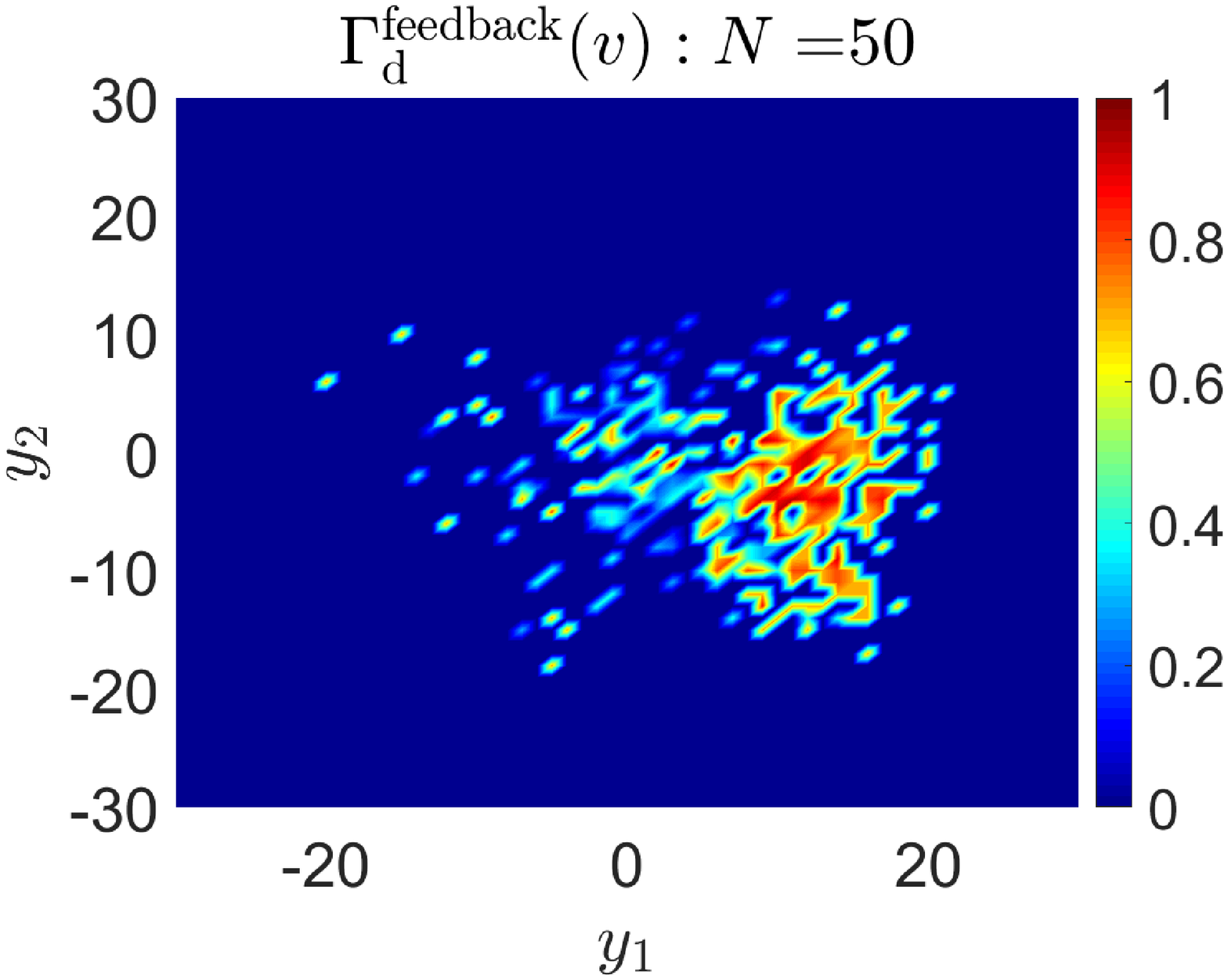}%
\label{fig.real_50}}
\hfil
\subfloat[]{\includegraphics[width=1.7in]{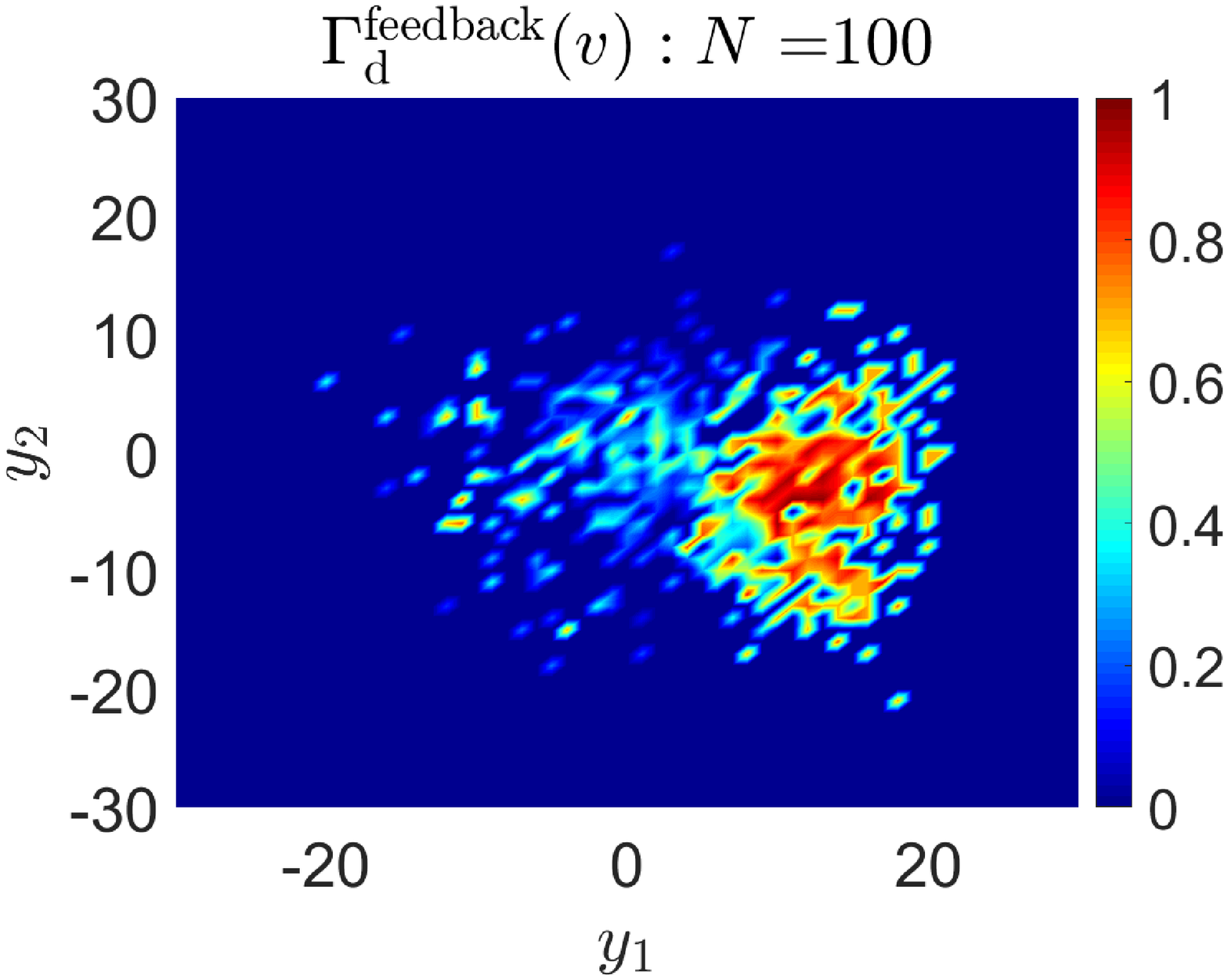}%
\label{fig.real_100}}
\hfil
\subfloat[]{\includegraphics[width=1.7in]{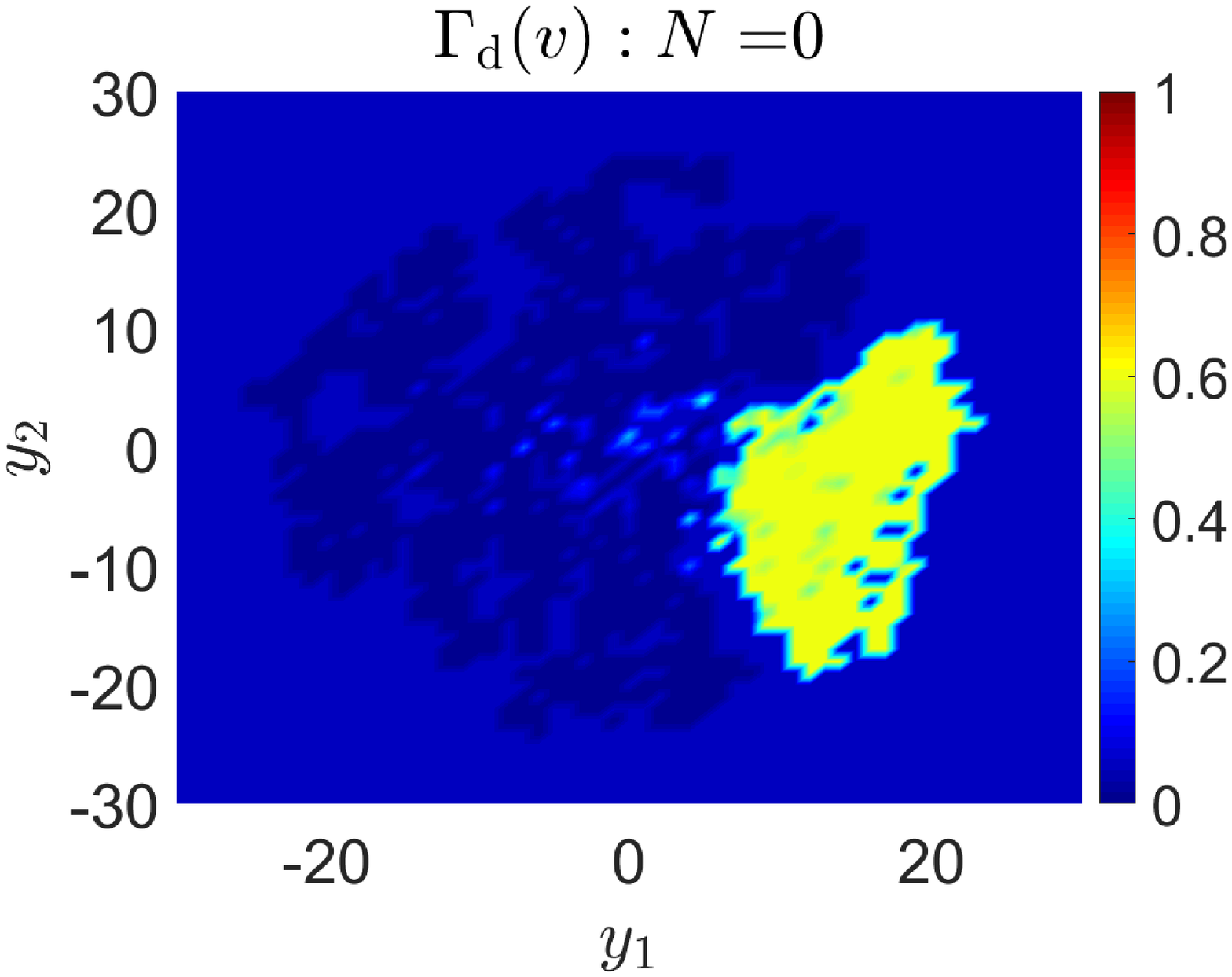}%
\label{fig.combine_0}}
\hfil
\subfloat[]{\includegraphics[width=1.7in]{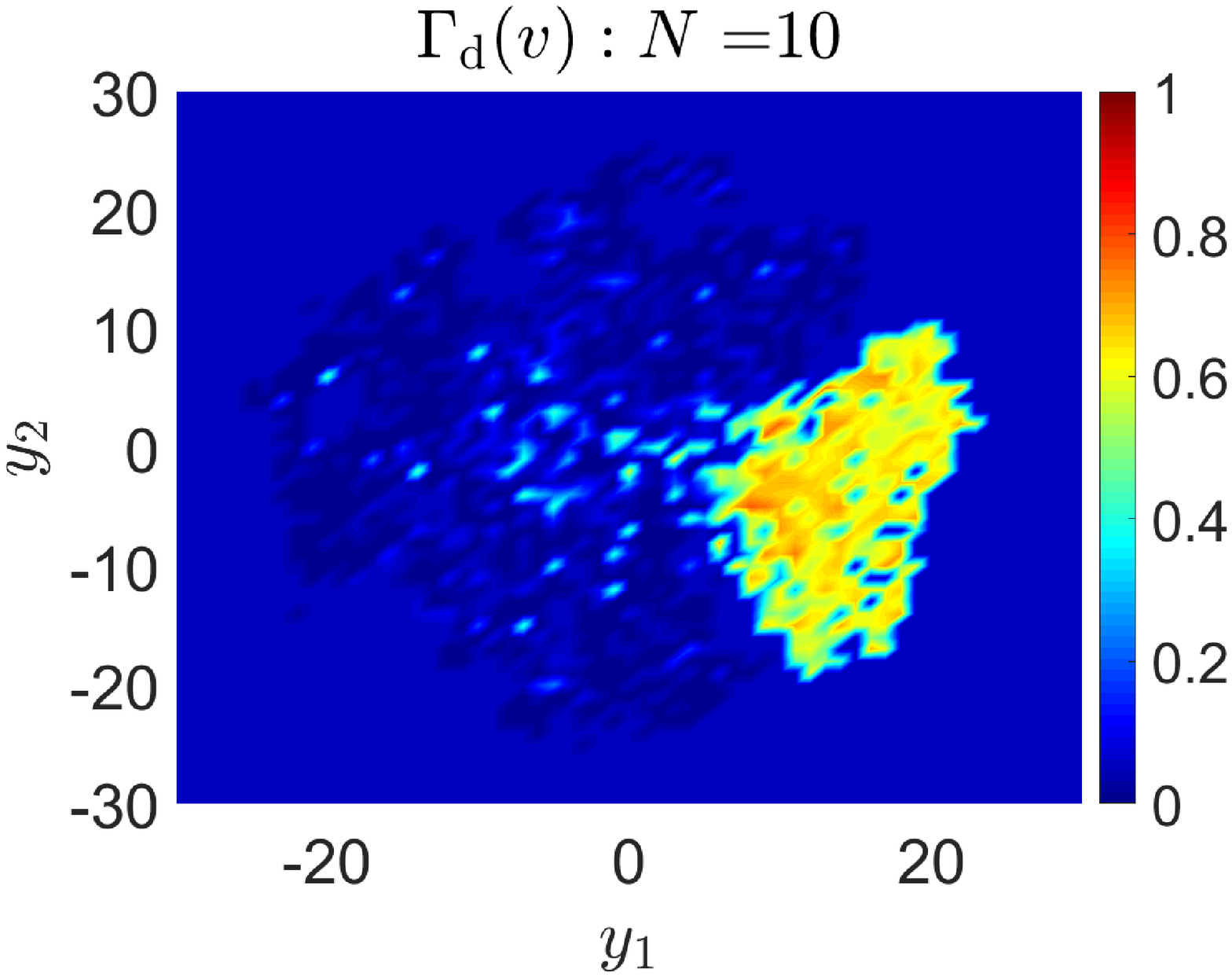}%
\label{fig.combine_10}}
\hfil
\subfloat[]{\includegraphics[width=1.7in]{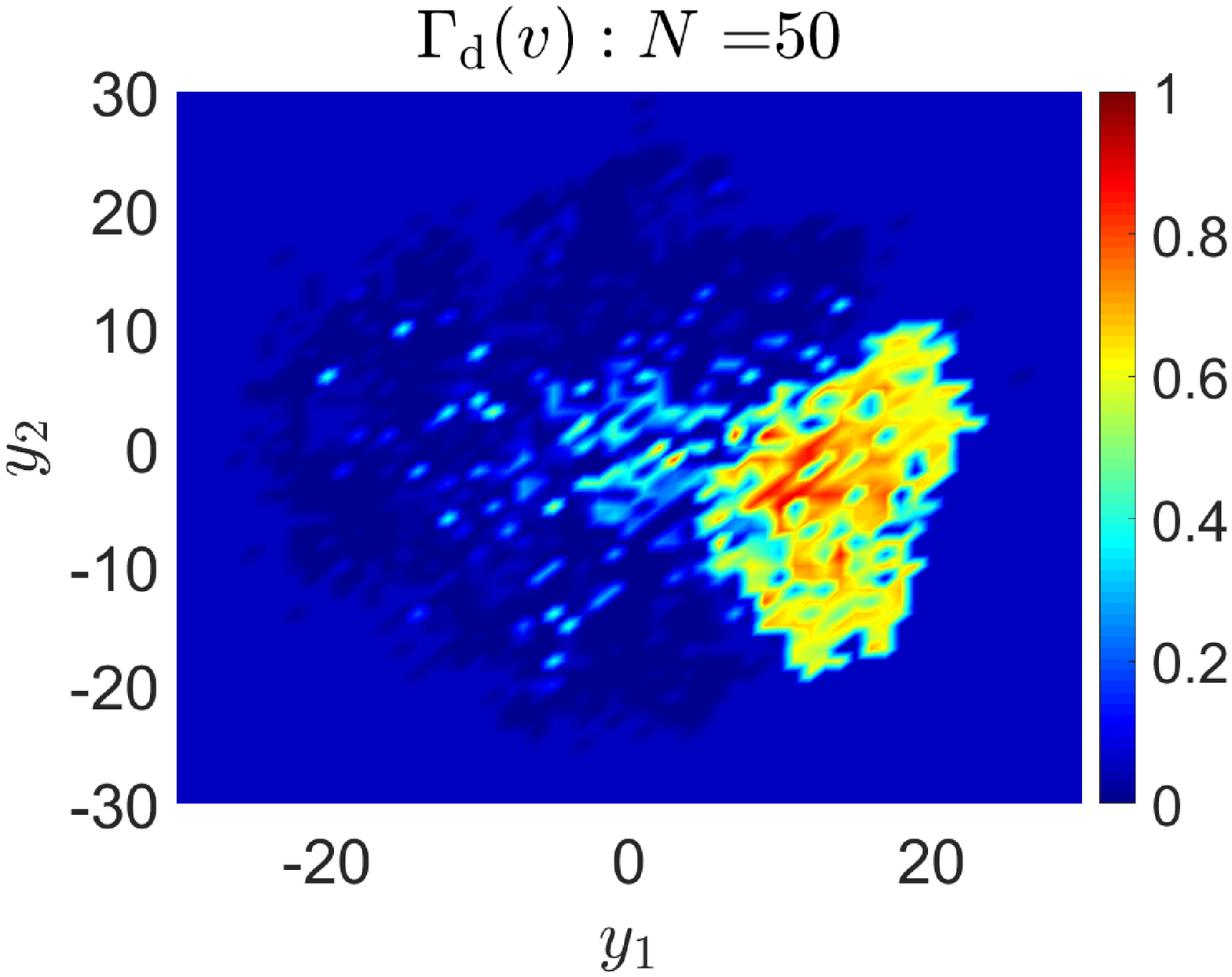}%
\label{fig.combine_50}}
\hfil
\subfloat[]{\includegraphics[width=1.7in]{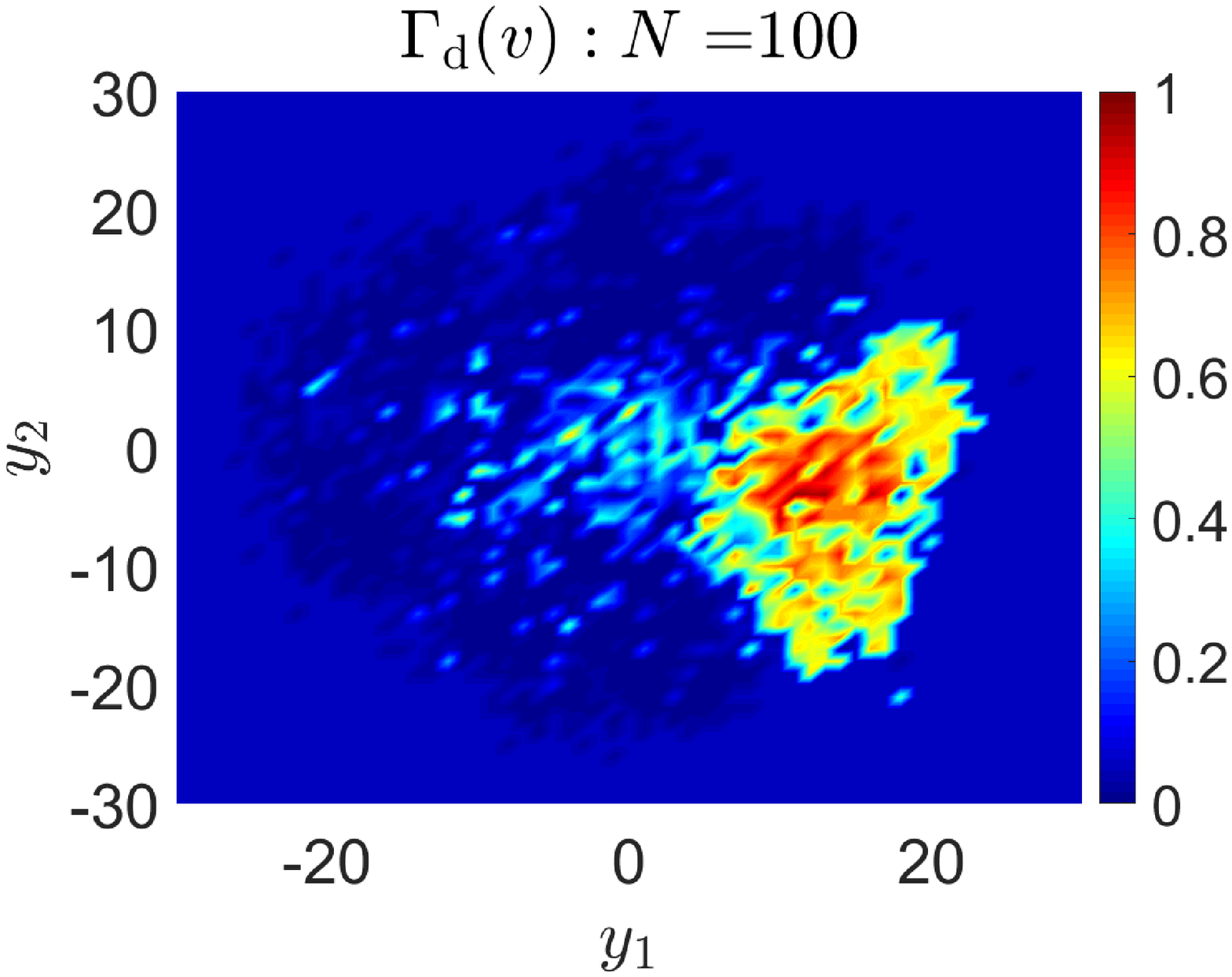}%
\label{fig.combine_100}}
\caption{Results of the online adaptation. (a)-(d) The prior DSAF $\Gamma_{\mathrm{d}}^{\mathrm{prior}}(v)$ in different update iterations $N$. $N = 0$ refers to the initialization prior to the online adaptation. The values of the safety estimates are represented by different colors.  (e)-(h) 
The feedback DSAF $\Gamma_{\mathrm{d}}^{\mathrm{feedback}}(v)$ in different update iterations $N$. (i)-(l) The DSAF $\Gamma_{\mathrm{d}}(v)$ in different update iterations $N$.}
\label{fig.online_update_results}
\end{figure*}

The initial realization of the low-dimensional safety feature, i.e., the values of simplified states $y^1,\ldots, y^{k_t}$, obtained from t-SNE is given in Fig.~\ref{fig.results_tSNE}. 
We use $\delta = 0.01$ in (\ref{eq.similarity_modified}) and set the perplexity and tolerance of t-SNE (see~\cite{maaten2008visualizing}) to 40 and $1e^{-4}$, respectively.
The result shows that the safe and unsafe original system states are clearly separated in the two-dimensional simplified state space $\mathcal{Y} \subseteq \mathbb{R}^2$. 

The state mapping $y = \Psi(x)$ is represented by a two-layer neural network with 128 neurons in each layer, which is trained using the initial realization of simplified states $y^1,\ldots, y^{k_t}$ and the set of original system states $\{x_{\mathrm{sim}}^1, \ldots, x_{\mathrm{sim}}^{k_t} \}$. 
By recomputing the outputs of the learned neural network, we obtain the final realization of the low-dimensional safety feature, i.e., the values of the simplified states $y^1,\ldots, y^{k_t}$, given in Fig.~\ref{fig.results_mapping_NN}.
Due to approximation error, certain simplified states have a slightly changed position compared to the values obtained from t-SNE.
However, this does not affect the computation of the low-dimensional representation of the safe region $\mathcal{S}_y$, as the results are updated later in the online adaptation using the feedback data.

We set the simplified state space as $\{\mathcal{Y} \hspace{1mm} | \hspace{1mm} -30 \leq y_1, y_2 \leq 30 \}$.
By discretizing the simplified state space $\mathcal{Y}$ into grid cells with step size 1 in both $y_1$ and $y_2$, we obtain the index vector $v \in \{1,2,\ldots,60 \}^2$. 
The prior DSAF $\Gamma_{\mathrm{d}}^{\mathrm{prior}}(v)$ is thus computed from the training dataset $\mathcal{D}_{\mathrm{train}}$ using the index vector $v$.
The results are given in Fig.~\ref{fig.sim_0}, where the initial subjective uncertainty, the initial estimate and the minimum number are selected as $\mu_{\mathrm{ini}} = 0.4$, $B_{\mathrm{ini}} = (0.05,0.55,0.4)$ and $k_{\mathrm{min}} =3$, respectively.
Depending on the number of safe and unsafe training data in each grid cell, the prior DSAF $\Gamma_{\mathrm{d}}^{\mathrm{prior}}(v)$ estimates the probability $p(x \in \mathcal{S})$ for original system states $x$ that take the index vector $v$ from the locating function $L(x)$.
In Fig.~\ref{fig.combine_0}, the DSAF $\Gamma_{\mathrm{d}}(v)$ is initialized by the prior DSAF $\Gamma_{\mathrm{d}}^{\mathrm{prior}}(v)$.
In the next subsection, we demonstrate the update process of the DSAF $\Gamma_{\mathrm{d}}(v)$ using the proposed online adaptation method.

\subsection{Updating the Low-dimensional Representation}
\label{sec.results_update}

To simulate a mismatch between the nominal and the real systems, we set the mass and the maximal lifting force of the real system to $m = 0.8 \text{ kg}$ and $f = 145 \text{ N}$, respectively.
To eliminate the influence of a specific learning task or algorithm and focus on illustrating the update process, the feedback dataset $\mathcal{D}_{\mathrm{feedback}}$ is obtained by randomly selecting states $x_{\mathrm{real}}$ where the corrective controller $K(x)$ is activated, such that the entire original system state space can be visited.

The following parameters are used in the online adaptation method: 
$\mu_{\mathrm{min}} = 0.1$, $p_{\mathrm{th}} = 0.3$, $\alpha = 3e^5$, $\beta = 0.3$, $\gamma = 0.4$.
The GPR model $\mathrm{GP}(x)$ uses a squared exponential kernel.
To demonstrate the online update process, we collect the feedback data one by one and incrementally extend the feedback dataset $\mathcal{D}_{\mathrm{feedback}}$.
The DSAF $\Gamma_{\mathrm{d}}(v)$ is updated once when every $k_u = 20$ feedback data are obtained.

\begin{figure*}[!t]
\centering
\subfloat[]{\includegraphics[width=1.7in]{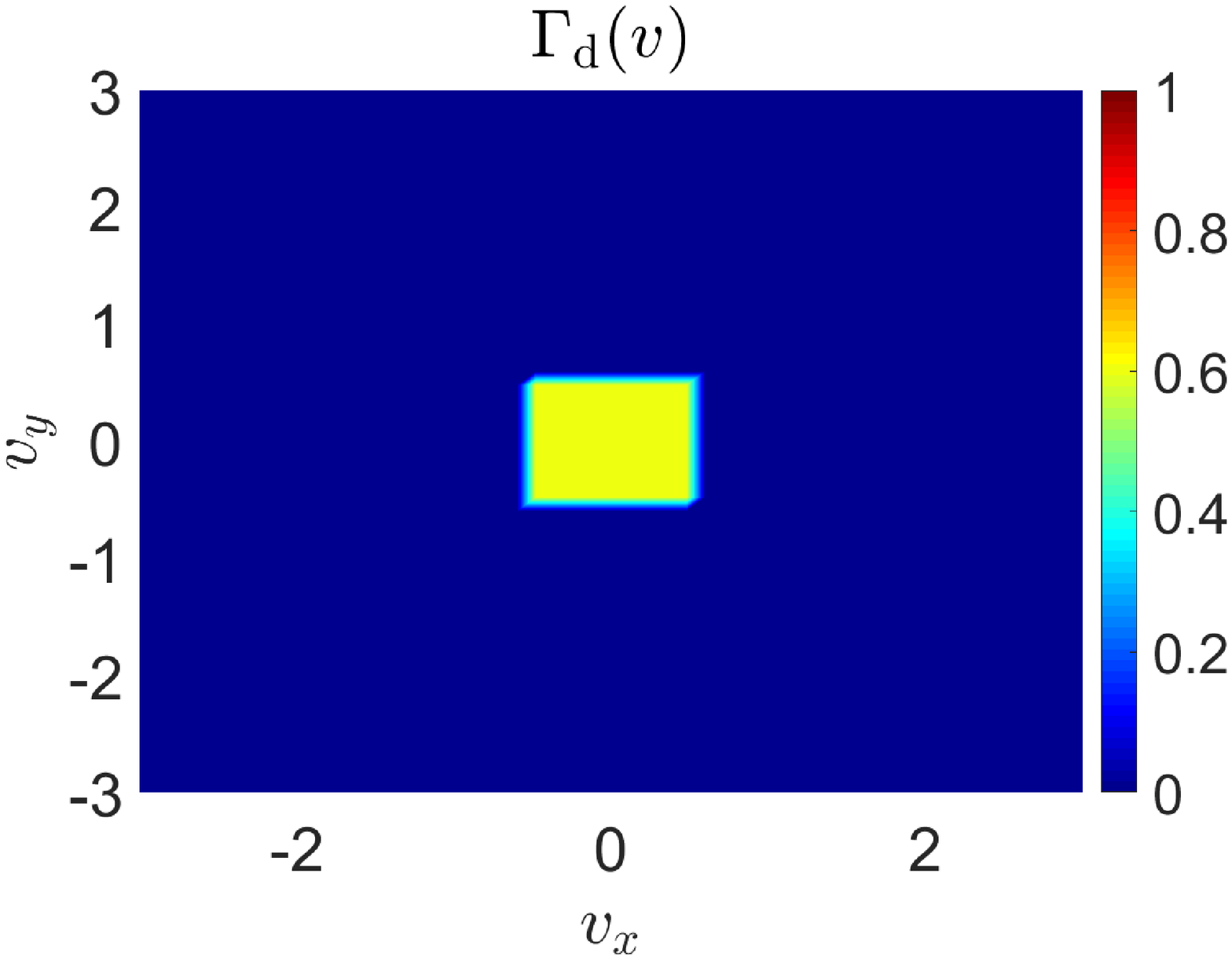}%
\label{fig.tro_ini}}
\hfil
\subfloat[]{\includegraphics[width=1.7in]{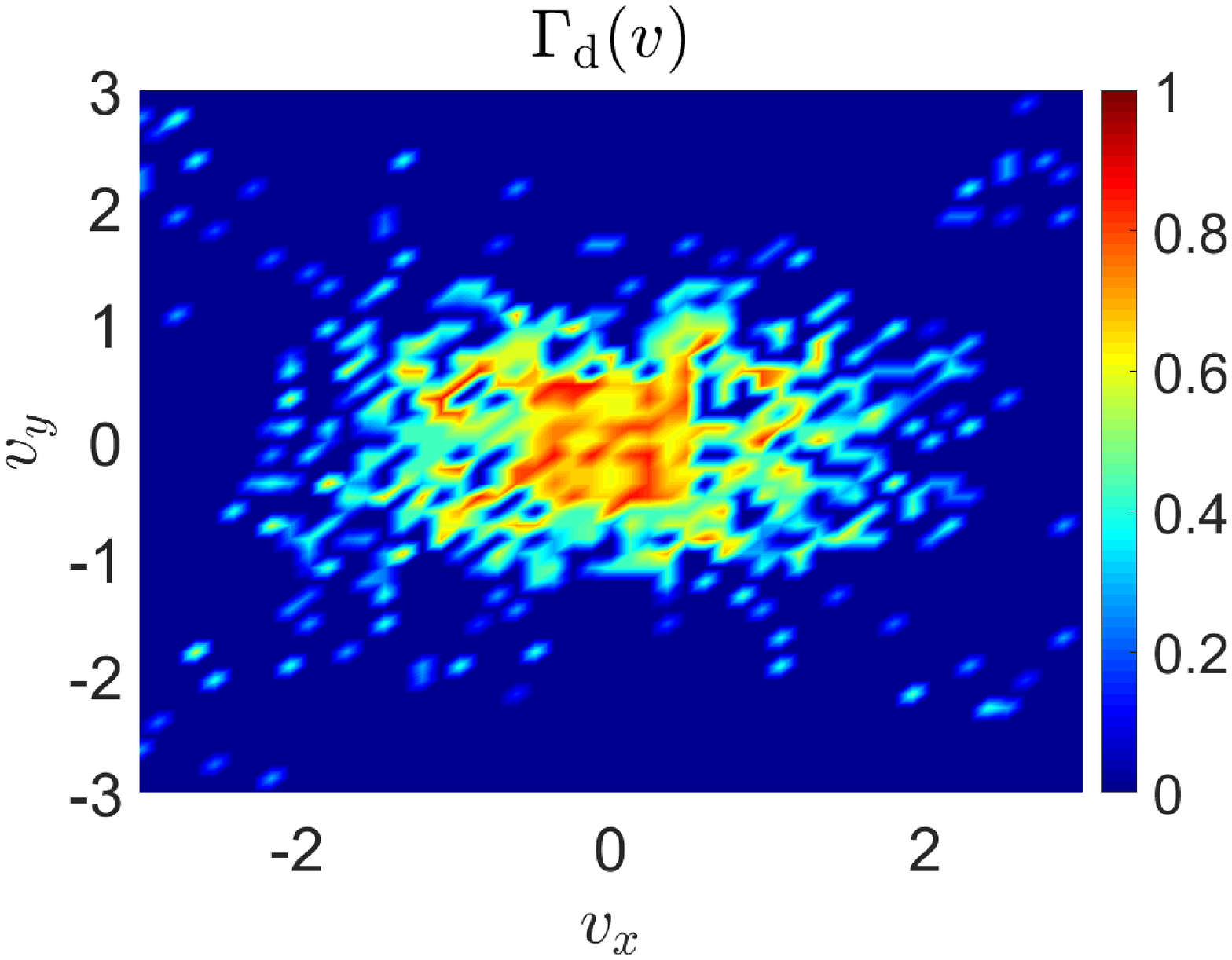}%
\label{fig.tro_small}}
\hfil
\subfloat[]{\includegraphics[width=1.7in]{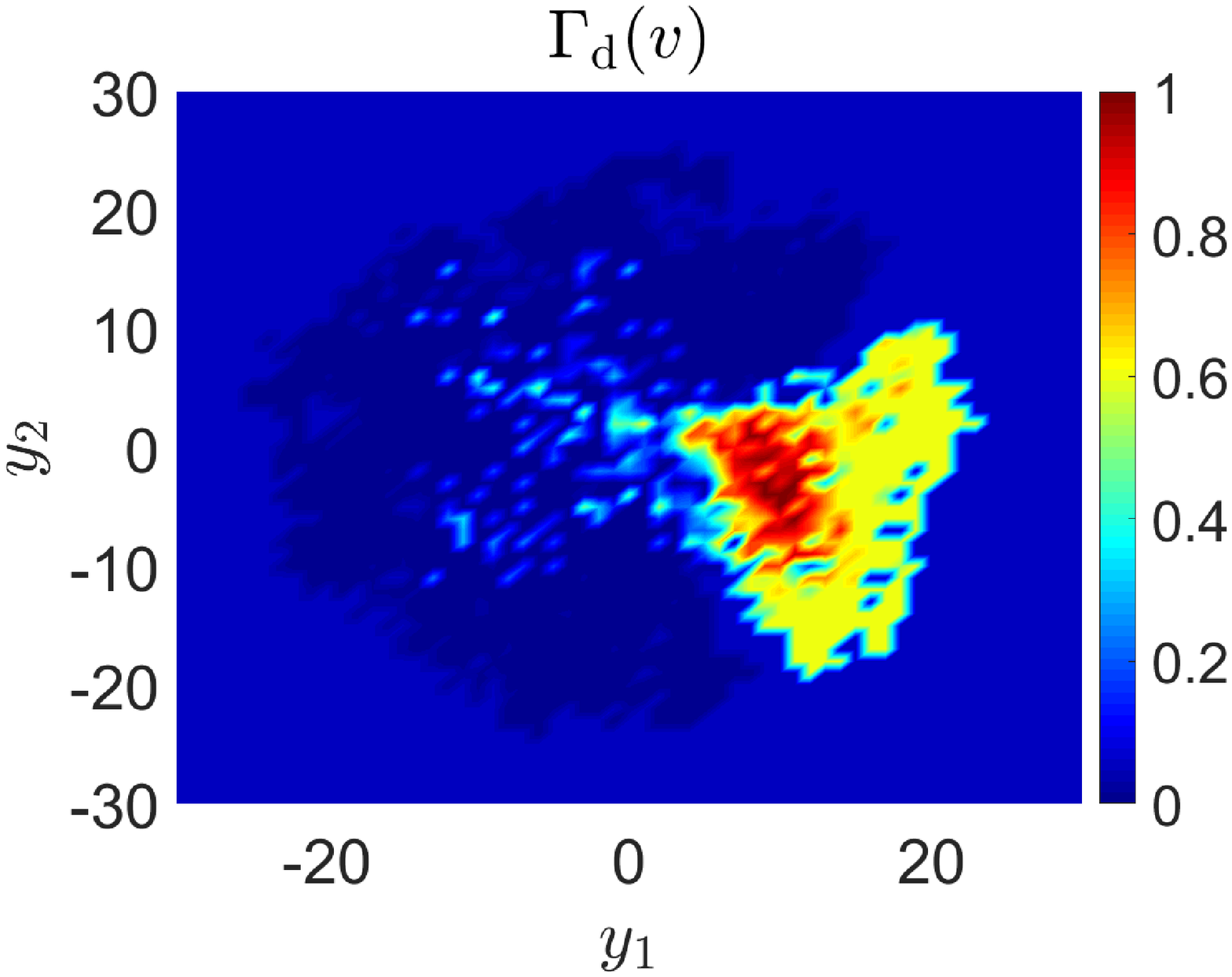}%
\label{fig.prop_small}}
\hfil
\subfloat[]{\includegraphics[width=1.7in]{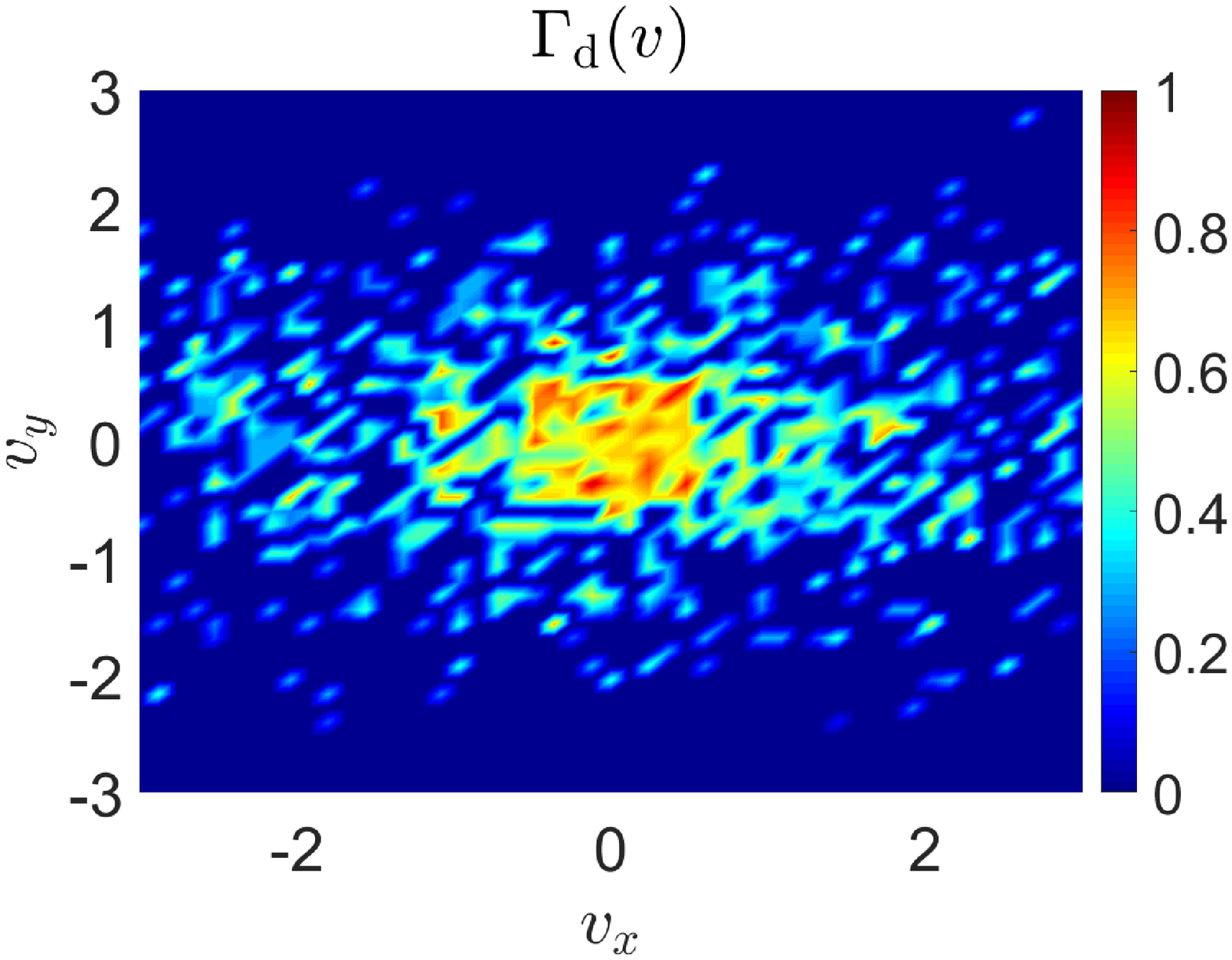}%
\label{fig.tro_full}}
\caption{Comparison with physically inspired model order reduction. (a) For physically inspired model order reduction, the DSAF $\Gamma_{\mathrm{d}}(v)$ is initialized conservatively. (b)-(c) The DSAFs $\Gamma_{\mathrm{d}}(v)$ obtained by using physically inspired model order reduction and the proposed approach, respectively. The feedback dataset $\mathcal{D}_{\mathrm{feedback}}^{'}$ is used for the update. (d) The DSAF $\Gamma_{\mathrm{d}}(v)$ obtained by using physically inspired model order reduction and the feedback dataset $\mathcal{D}_{\mathrm{feedback}}$. }
\label{fig.compare_results}
\end{figure*}

The results of the online adaptation are given in Fig.~\ref{fig.online_update_results}.
Prior to the update (update iteration $N = 0$), the DSAF $\Gamma_{\mathrm{d}}(v)$ is initialized as the prior DSAF $\Gamma_{\mathrm{d}}^{\mathrm{prior}}(v)$, while the feedback DSAF $\Gamma_{\mathrm{d}}^{\mathrm{feedback}}(v)$ is constructed using the empty BBA $B_{\varnothing}$ (see Fig.~\ref{fig.sim_0},~\ref{fig.real_0},~\ref{fig.combine_0}).
Once the learning procedure has started, we collect the feedback data incrementally.
In the early updating phase, e.g., update iteration $N = 10$, the DSAF $\Gamma_{\mathrm{d}}(v)$ is mainly determined by the prior DSAF $\Gamma_{\mathrm{d}}^{\mathrm{prior}}(v)$. 
The subjective uncertainties of each training data are modified using the feedback data, where we become confident about the safety of certain training data when we observe a nearby feedback data that has the same safety property. 
Since the amount of feedback data is insufficient for providing a reliable safety estimate, the feedback DSAF $\Gamma_{\mathrm{d}}^{\mathrm{feedback}}(v)$ has a smaller effect on the computation of the low-dimensional representation of the safe region $\mathcal{S}_y$ (see Fig.~\ref{fig.sim_10},~\ref{fig.real_10},~\ref{fig.combine_10}).

When more feedback data are available, e.g., update iteration $N = 50$, the feedback DSAF $\Gamma_{\mathrm{d}}^{\mathrm{feedback}}(v)$ is able to provide more accurate safety estimates, hence its influence on the DSAF $\Gamma_{\mathrm{d}}(v)$ also becomes more significant.
Due to the high dimensionality of the original system state $x$ and the limited amount of feedback data, it is difficult to acquire an estimate with high confidence from the GPR model $\mathrm{GP}(x)$. 
As a result, changes are marginal in the prior DSAF $\Gamma_{\mathrm{d}}^{\mathrm{prior}}(v)$ (see Fig.~\ref{fig.sim_50},~\ref{fig.real_50},~\ref{fig.combine_50}).
With even more feedback data, e.g., update iteration $N = 100$, the DSAF $\Gamma_{\mathrm{d}}(v)$ is able to provide reliable estimates about the probability $p(x \in \mathcal{S})$ for each index vector $v$.
While the prior and feedback DSAFs are updated accordingly, the DSAF $\Gamma_{\mathrm{d}}(v)$ represents the actual low-dimensional representation of the safe region $\mathcal{S}_y$ under the unknown part of the system dynamics $d(x)$ (see Fig.~\ref{fig.sim_100},~\ref{fig.real_100},~\ref{fig.combine_100}).

\subsection{Comparison with Physically Inspired Model Order Reduction}
\label{sec.compare_results}
We compare the proposed approach with the physically inspired model order reduction presented in~\cite{zhou2020general} in terms of the representation power of the identified low-dimensional representation of the safe region $\mathcal{S}_y$, i.e., how well the safe and unsafe states are separated.
To do this, we compute another DSAF $\Gamma_{\mathrm{d}}(v)$ using physical features.  
As in~\cite{zhou2020general}, the low-dimensional safety feature, i.e., the simplified state $y$, is selected for the velocities in $x$ and $y$ directions $y = [v_x, v_y]^T$.
To avoid any dangerous behavior in early learning phase, the low-dimensional representation of the safe region $\mathcal{S}_y$ is initialized conservatively~\cite{zhou2020general} by setting $\Gamma_{\mathrm{d}}(v) = 0.6$ for grid cells that satisfy $-0.5 \leq v_x, v_y \leq 0.5$ (see Fig.~\ref{fig.tro_ini}).

As the learning task in~\cite{zhou2020general} is relatively simple, the exploration in the original system state space is limited to a small subspace around the origin (see Section~\ref{sec.discussion_tasks} for more discussions on this point).
Therefore, to make a fair comparison, we also generate another feedback dataset $\mathcal{D}_{\mathrm{feedback}}^{'}$ that has the same size as the dataset $\mathcal{D}_{\mathrm{feedback}}$. 
However, instead of the complete original system state space given in Section~\ref{sec.results_setup}, the states $x_{\mathrm{real}}$ in the set $\mathcal{D}_{\mathrm{feedback}}^{'}$ are sampled from a smaller state space, where the ranges of angular positions and angular velocities are changed to $ -\frac{1}{3}\pi \leq \theta_r, \theta_p, \theta_y \leq \frac{1}{3}\pi \text{ rad}$ and $-3 \text{ rad/s} \leq \omega_r, \omega_p, \omega_y \leq 3 \text{ rad/s}$, respectively.

We first compare the performance of both approaches by considering a small state space, i.e., the feedback dataset $\mathcal{D}_{\mathrm{feedback}}^{'}$ is used for the update. 
The results show that, in this case, physical features are able to provide reasonable predictions about safety, i.e., the safe and unsafe regions are separated (see Fig.~\ref{fig.tro_small}).
Meanwhile, the proposed approach also produces a satisfying result with a marginally better separation between safe and unsafe states (see Fig.~\ref{fig.prop_small}). 

However, if the learning task becomes more complex, the complete state space usually has to be explored to enable an optimal policy to be found.
To simulate this scenario, we also update the initial DSAF $\Gamma_{\mathrm{d}}(v)$ using the feedback dataset $\mathcal{D}_{\mathrm{feedback}}$. 
As seen in Fig.~\ref{fig.tro_full}, when considering the entire original system state space, it is difficult to make reliable safety estimates based only on physical features.
The boundary between safe and unsafe regions becomes unclear, and there are numerous grid cells that lead to a safety estimate close to $0.5$.  
In contrast, the proposed approach is still able to find a representative low-dimensional representation of the safe region $\mathcal{S}_y$ for the complete state space.  
As the identified simplified state $y$ can describe the safety of original system states $x$ more precisely, a satisfying separation between safe and unsafe regions is achieved (see Fig.~\ref{fig.combine_100}) and more useful safety estimates are obtained.
The independence of the size of the state space indicates the possibility of implementing the proposed approach on different learning tasks, which in turn increases the applicability of the SRL framework.

\section{Discussion}
\label{sec.discussion}
In this work, we propose a general approach for efficiently identifying a low-dimensional representation of the safe region.
Two important aspects of the proposed approach are discussed in this section.

\subsection{Relevance to Different SRL Tasks}
\label{sec.discussion_tasks}

In~\cite{zhou2020general}, the SRL framework utilizes the low-dimensional representation of the safe region $\mathcal{S}_y$ that is obtained using physically inspired model order reduction.
Such a low-dimensional representation is useful when the learning task is relatively simple, e.g., teaching a quadcopter to fly forwards as given in~\cite{zhou2020general}, such that a satisfying control policy can be found without requiring an extensive exploration in the original state space. 
Since, in this case, the system state is likely to stay in a sub-state space near the origin, physical features are able to provide reliable safety estimates. 
However, when the learning task becomes more difficult, e.g., the quadcopter needs to track a complex 3D trajectory, the learning algorithm in general has to explore a large portion of the state space to find an optimal policy. 
Under these circumstances, at least a rough safety assessment of the complete state space is needed.
Unfortunately, being restricted by the representation power, the physically inspired low-dimensional representation of the safe region $\mathcal{S}_y$ fails to provide useful safety estimates when considering the entire state space.
Hence, the performance of the SRL framework is affected.

Therefore, to overcome this problem, this paper proposes a data-driven approach for identifying a low-dimensional representation of the safe region $\mathcal{S}_y$ that is able to make more precise predictions about safety.
Meaningful safety estimates are even obtained for the entire original state space.
This not only gives the learning algorithm more flexibility in choosing its actions to find the optimal policy, but also indicates the applicability of the proposed approach to more complex learning tasks.

\subsection{Strengths and Limitations}

The presented approach has three particular strengths.
First, it finds a low-dimensional representation of the safe region $\mathcal{S}_y$ that allows safe and unsafe states to be clearly separated for large portions of a high-dimensional state space; see also Section~\ref{sec.compare_results}. 
Second, the effort required for identifying the low-dimensional representation of the safe region $\mathcal{S}_y$ is low.
While, for instance, physically inspired model order reduction usually needs a comprehensive analysis of the system dynamics, the proposed approach relies solely on training data that can be collected efficiently even for complex dynamical systems through parallel computing and a suitable simulation environment.
Third, it fully utilizes the information contained in the feedback data using two DSAFs.
Hence, the update can be performed with few feedback data while providing a satisfying result.

However, the performance of the identified low-dimensional representation of the safe region $\mathcal{S}_y$ is affected by the quality of the nominal system, i.e. the magnitude of the discrepancy between the nominal and the real systems.
While the state mapping $y = \Psi(x)$ is determined using only training data, the online adaptation method attempts to find an accurate DSAF $\Gamma_{\mathrm{d}}(v)$ based on the learned low-dimensional safety feature.
If the reality gap is too large, then it is possible that the learned safety feature is not sufficiently representative and we might therefore observe more grid cells with final safety estimates that are close to 0.5, i.e., $\Gamma_{\mathrm{d}}(v) \approx 0.5$, which are less useful for guiding the learning process.
In general, if the nominal system is assumed to be unreliable, a high probability threshold $p_t$ should be used for constructing the low-dimensional representation of the safe region $\mathcal{S}_y$ (see \eqref{eq.map_m}), such that the learning process becomes more conservative for keeping the system safe.
However, we usually consider the unknown system dynamics $d(x)$ as bounded within a reasonable range, since it makes less sense to use a dissimilar nominal system to predict the behavior of the real system.
To further generalize the proposed approach, more studies are required to quantify the influence of the simulation-to-reality gap on the reliability of the obtained safety estimates.

\section{Conclusion}
\label{sec.conclusion}

To apply SRL to complex dynamical systems, this paper proposes a novel data-driven approach to identify a low-dimensional representation of the safe region for realizing a general SRL framework. 
Using a nominal system model that predicts the behavior of the real system, we first collect training data about the safety of different system states.
Then, by computing the probabilistic similarities between each training data using a data-driven method, an initial low-dimensional representation of the safe region is obtained.
To compensate for the mismatch between the nominal and the real systems, an efficient online adaptation method based on belief function theory is also proposed to update the low-dimensional representation of the safe region by accounting for the real system behavior.
Experimental results show that, compared to the previous work, a more reliable and representative low-dimensional representation of the safe region is found using the proposed approach.
However, our approach has the limitation that its performance is affected by the magnitude of discrepancy between the nominal and real systems.
If the reality gap is assumed to be large, then it is likely that a less meaningful low-dimensional representation of the safe region will be obtained.

For future work, we intend to combine the data-driven method with model-based model order reduction techniques to find an approach that is more robust to the simulation-to-reality gap when identifying the low-dimensional representation of the safe region.
Moreover, we also plan to investigate the possibility of quantifying the similarity between different dynamical systems, such that the learned safety feature can be generalized from one system to other similar systems.
How the similarity between dynamical systems will be measured, is, however, still an open research problem.

\ifCLASSOPTIONcaptionsoff
  \newpage
\fi



\bibliographystyle{IEEEtran}
\bibliography{ref.bib}

\begin{IEEEbiography}[{\includegraphics[width=1in,height=1.25in,clip,keepaspectratio]{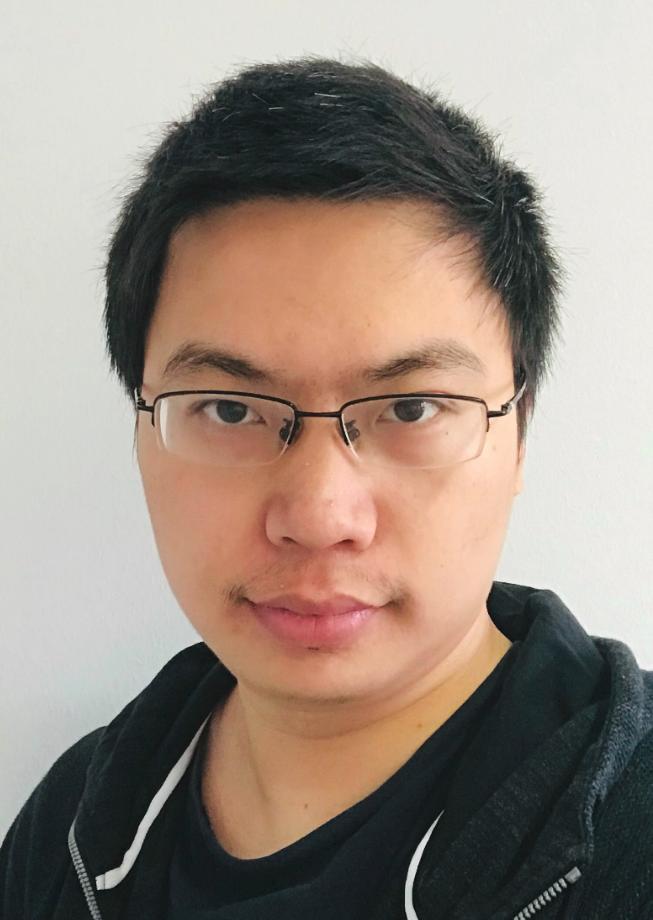}}]{Zhehua Zhou}
received the B.E. degree in mechatronics engineering from Tongji University, Shanghai, China in 2014 and the M.Sc. degree in electrical and computer engineering from the Department of Electrical and Computer Engineering, Technical University of Munich, Munich, Germany in 2017.
He is currently working toward the Ph.D. degree in learning-based control and robotics from the Chair of Automatic Control Engineering, Department of Electrical and Computer Engineering, Technical University of Munich, Munich, Germany. 
His research interests include optimal control, learning-based control and the applications to robotics.
\end{IEEEbiography}

\begin{IEEEbiography}[{\includegraphics[width=1in,height=1.25in,clip,keepaspectratio]{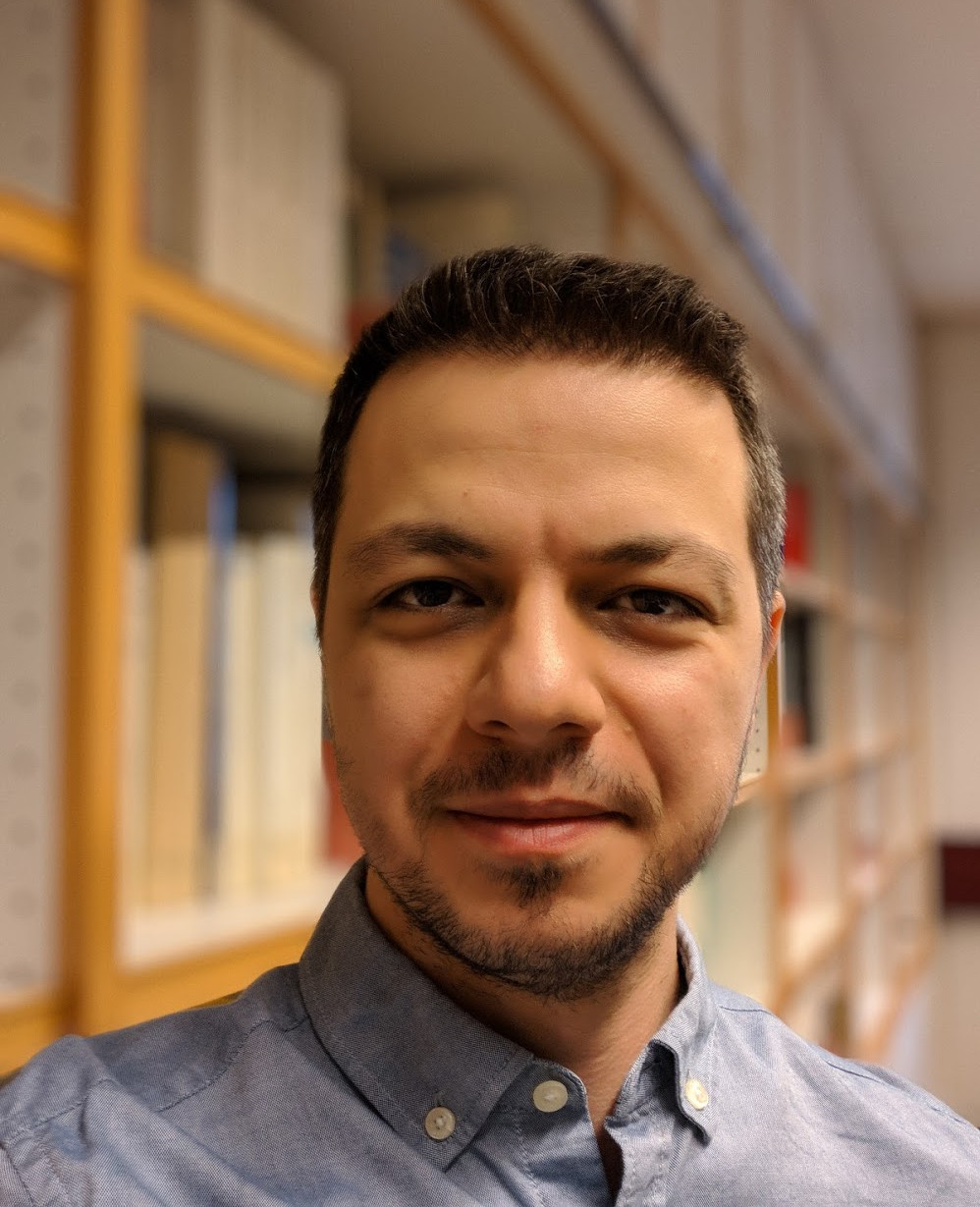}}]{Ozgur S. Oguz} received the B.Sc. and the M.Sc. (summa cum laude) degrees in computer science from Koc University, Istanbul, Turkey, and the Ph.D. degree (summa cum laude) from the Department of Electrical and Computer Engineering, Technical University of Munich, Germany, in 2018. Currently he is a postdoctoral researcher at the Machine Learning and Robotics Lab, University of Stuttgart and Max Planck Institute for Intelligent Systems, Germany. His research interests are developing autonomous systems that are able to reason about their states of knowledge, take sequential decisions to realize a goal, and simultaneously learn to improve their causal physical reasoning and manipulation skills.
\end{IEEEbiography}

\begin{IEEEbiography}[{\includegraphics[width=1in,height=1.25in,clip,keepaspectratio]{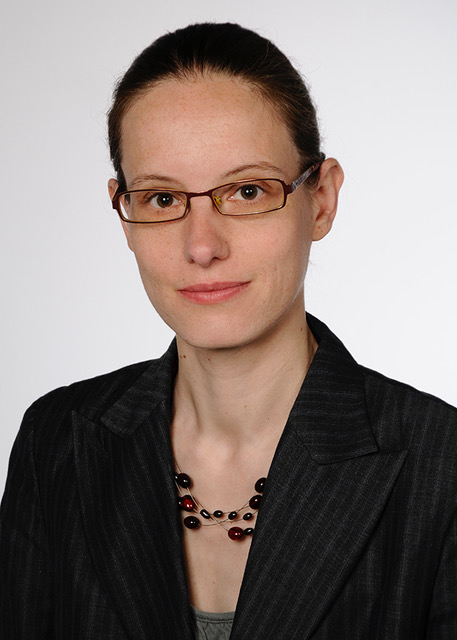}}]{Marion Leibold (nee Sobotka)}
received the diploma degree in applied mathematics from the Technische Universit{\"a}t M{\"u}nchen, Germany, in 2002. 
She received the PhD degree from the Faculty of Electrical Engineering and Information Technology, Technische Universit{\"a}t M{\"u}nchen, Germany, in 2007. 
Currently she is a senior researcher at the Institute of Automatic Control Engineering, Faculty of Electrical Engineering and Information Technology, Technische Universit{\"a}t M{\"u}nchen, Germany. 
Her research interests include optimal control and nonlinear control theory, and the applications to robotics. 
\end{IEEEbiography}

\begin{IEEEbiography}[{\includegraphics[width=1in,height=1.25in,clip,keepaspectratio]{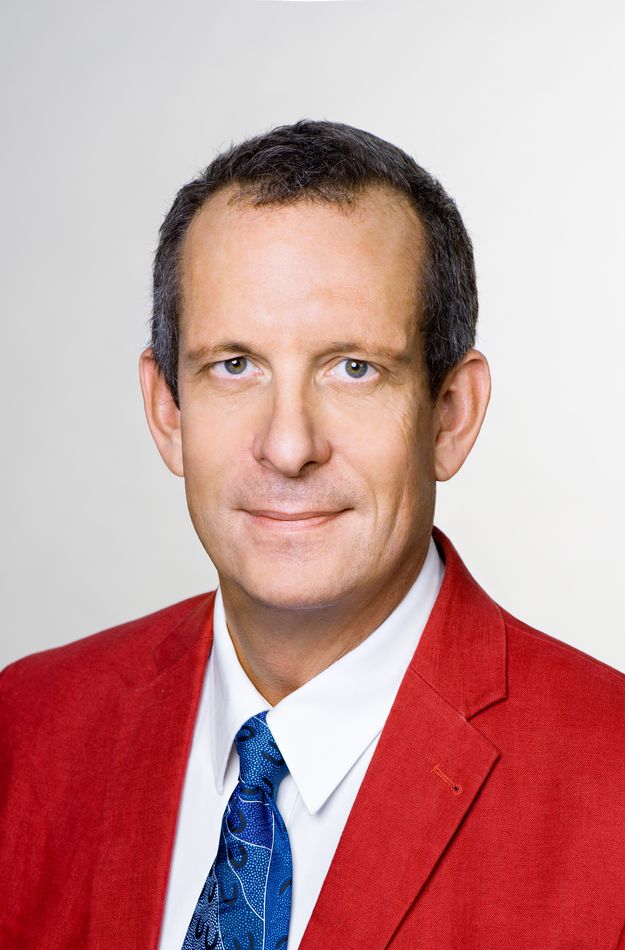}}]{Martin Buss}
received the Diploma Engineering degree Technische Universit{\"a}t Darmstadt, Darmstadt, Germany, in 1990 and the Doctor of Engineering degree from The University of Tokyo, Tokyo, Japan, in 1994, both in electrical engineering.
In 1988, he was a Research Student for one year with Science University of Tokyo.
From 1994 to 1995, he was a Postdoctoral Researcher in the Department of Systems Engineering, Australian National University, Canberra, ACT, Australia. 
From 1995 to 2000, he was a Senior Research Assistant and Lecturer in the Chair of Automatic Control Engineering, Department of Electrical Engineering and Information Technology, Technical University of Munich, Munich, Germany.
From 2000 to 2003, he was a Full Professor, the Head of the Control Systems Group, and the Deputy Director of the Institute of Energy and Automation Technology, Faculty IV, Electrical Engineering and Computer Science, Technical University Berlin, Berlin, Germany.
Since 2003, he has been a Full Professor (Chair) in the Chair of Automatic Control Engineering, Faculty of Electrical Engineering and Information Technology, Technical University of Munich, where he has been in the Medical Faculty since 2008.
Since 2006, he has also been the Coordinator of the Deutsche Forschungsgemeinschaftcluster of excellence ``Cognition for Technical Systems (CoTeSys)", Bonn, Germany. 
His research interests include automatic control, mechatronics, multimodal human system interfaces, optimization, nonlinear, and hybrid discrete-continuous systems.
\end{IEEEbiography}

\end{document}